%% file: iclr2027_conference.tex
\documentclass{article} 
\usepackage{iclr2027_conference,times}

\iclrfinalcopy
\input{math_commands.tex}

\usepackage{hyperref}
\usepackage{url}
\usepackage{amsmath}
\usepackage{amssymb}
\usepackage{amsthm}
\usepackage{amsfonts}
\usepackage{mathtools}
\usepackage{booktabs}
\usepackage{array}
\usepackage{microtype}
\usepackage{algorithm}
\usepackage{algpseudocode}
\usepackage{graphicx}
\usepackage{wrapfig}
\usepackage{lipsum}
\usepackage{titletoc}
\usepackage{enumerate}
\usepackage{multirow}
\usepackage{graphicx}
\usepackage[table]{xcolor}
\usepackage[textsize=scriptsize,color=cyan!15,bordercolor=cyan!60!black]{todonotes}

\newcommand{\refig}[1]{Fig.~\ref{#1}}
\newcommand{\refapp}[1]{App.~\ref{#1}}

\newcommand{\refsec}[1]{Sec.~\ref{#1}}

\newcommand{\Sset}{\mathcal{S}}

\DeclareMathOperator{\LSE}{LSE}

\newtheorem{theorem}{Theorem}

\newtheorem{proposition}{Proposition}
\newtheorem{corollary}{Corollary}
\newtheorem{remark}{Remark}
\newtheorem{definition}{Definition}

\providecommand{\KL}{\mathrm{KL}}

\ifcsname theorem\endcsname\else
  \newtheorem{theorem}{Theorem}
\fi
\ifcsname proposition\endcsname\else
  \newtheorem{proposition}[theorem]{Proposition}
\fi
\ifcsname observation\endcsname\else
  \newtheorem{observation}[theorem]{Observation}
\fi
\ifcsname lemma\endcsname\else
  \newtheorem{lemma}[theorem]{Lemma}
\fi
\ifcsname corollary\endcsname\else
  \newtheorem{corollary}[theorem]{Corollary}
\fi
\ifcsname definition\endcsname\else
  \newtheorem{definition}[theorem]{Definition}
\fi
\ifcsname assumption\endcsname\else
  
\fi
\ifcsname remark\endcsname\else
  \newtheorem{remark}[theorem]{Remark}
\fi
 
\newcommand{\Rbar}{\mathbb{R}\cup\{-\infty\}}
\newcommand{\Yp}{\mathcal{Y}^{+}_{x}}
\newcommand{\Cp}{\mathcal{C}^{+}}

\newcommand{\Ppi}{P_{\bar\pi}}
\newcommand{\VER}{V^{\bar\pi}}
\newcommand{\Vtrop}{V^{\bar\pi}_{\mathrm{trop}}}
\newcommand{\Qtrop}{Q^{\bar\pi}_{\mathrm{trop}}}
\newcommand{\Atrop}{A^{\bar\pi}_{\mathrm{trop}}}
\newcommand{\Vcap}{V_{\mathrm{cap}}}

\newcommand{\Jtrop}{J_{\mathrm{trop}}}
\newcommand{\JER}{J_{\mathrm{ER}}}
\newcommand{\JML}{J_{\mathrm{ML}}}

\newcommand{\basis}{\mathcal{B}}

\makeatletter
\patchcmd{\proof}
  {\topsep6\p@\@plus6\p@\relax}
  {\topsep0pt\relax}
  {}{}
\makeatother

\title{Tropical Reinforcement Learning}
\author{
  \hspace{-1.5em}\textbf{Arip Asadulaev}$^{\dagger}$ %
  \textbf{Aladin Djuhera}$^{\ddagger}$ %
  \textbf{Karim Salta}$^{\dagger}$ %
  \textbf{Holger Boche}$^{\ddagger}$ %
  \textbf{Fakhri Karray}$^{\dagger}$ %
  \textbf{Martin Takac}$^{\dagger}$%
}

\definecolor{skyblueA}{HTML}{EAF6FF}   
\definecolor{skyblueB}{HTML}{CFEAFF}   
\definecolor{skyblueC}{HTML}{4FA3E3}   
\definecolor{skyblueD}{HTML}{1D70B8}   
\definecolor{skyblueE}{HTML}{0F3D63}   
\definecolor{anthropicorange}{RGB}{204, 120, 92}
\definecolor{neuorange}{HTML}{E64A19}   

\usepackage[most]{tcolorbox}
\usepackage{xcolor}
\usepackage[scaled]{helvet}

\definecolor{takeawayborder}{HTML}{111111} 
\definecolor{takeawaytitlebg}{HTML}{7A0F14} 
\definecolor{takeawaybg}{HTML}{FFFFFF}      
\usepackage[T1]{fontenc}
\usepackage[scale=0.99]{sourcecodepro}      

\newcommand{\ccfont}{\ttfamily}
 
\newcommand{\pmv}[2]{#1\,$\pm$\,#2}
\newcommand{\pmvb}[2]{\textbf{#1\,{\boldmath$\pm$}\,#2}}

\hypersetup{
    colorlinks=true,
    citecolor=anthropicorange,     
    linkcolor=red,      
    urlcolor=anthropicorange       
}

\newtcolorbox{takeawaybox}[1][]{
  enhanced,
  breakable,
  colback=takeawaybg,
  colframe=takeawayborder,
  boxrule=1.1pt,
  arc=2.2mm,
  left=2mm,
  right=2mm,
  top=1.3mm,
  bottom=1.3mm,
  title=Takeaway,
  coltitle=white,
  fonttitle=\bfseries\fontfamily{phv}\selectfont,
  fontupper=\normalfont,
  attach boxed title to top left={xshift=2mm,yshift=-2mm},
  boxed title style={
    colback=anthropicorange,
    colframe=anthropicorange,
    arc=1.3mm,
    boxrule=0pt
  },
  before skip=8pt,
  after skip=10pt,
  #1
}

\begin{document}

\maketitle
\begingroup
  \renewcommand\thefootnote{}
  \footnotetext{\texttt{arip.asadulaev@mbzuai.ac.ae}\quad $^{\dagger}$MBZUAI \quad $^{\ddagger}$TUM}%
\endgroup

\input{chapters/0_abstract}
\input{chapters/1_introduction_3}

\input{chapters/preliminaries}
\input{chapters/antropical_sections_3}
\input{chapters/antropical_sections_4}
\input{chapters/4_related_work}
\input{chapters/results}


\bibliography{bibliography}
\bibliographystyle{iclr2027_conference}


\newpage

\appendix
\input{appendix/appendix_antropical}
\input{appendix/appendix_environments}
\input{appendix/appendix_training_setup}
\input{appendix/appendix_limitations}

\end{document}

%% file: math_commands.tex
\usepackage{amsmath,amsfonts,bm}

\def\eqref#1{equation~\ref{#1}}

\def\1{\bm{1}}

\DeclareMathAlphabet{\mathsfit}{\encodingdefault}{\sfdefault}{m}{sl}
\SetMathAlphabet{\mathsfit}{bold}{\encodingdefault}{\sfdefault}{bx}{n}

\newcommand{\KL}{D_{\mathrm{KL}}}



%% file: chapters/0_abstract.tex
\begin{abstract}
    Reinforcement learning for large language models typically maximizes expected return, adding up the probabilities of all successful trajectories.
    However, the classical sum formulation can only report \emph{how often the model policy succeeds, not which solution actually worked}, and because probabilities sum to one, reinforcing one solution can make the model forget another that was never shown to be wrong.
    This makes expected return a poor fit for compositional reasoning, where a solution must be assembled from reasoning steps that the model produces in separate, often failed, attempts but rarely produces together.
    To address this, we propose Tropical Reinforcement Learning, which rests on a simple change of algebra: instead of adding the probabilities of alternative solutions, we take their maximum, which yields the \emph{tropical semiring}.
    The value of a state then becomes the log-probability of its most likely verified solution, together with an explicit path that can be replayed and reused.
    This enables true composition, since the best prefix and the best suffix meeting at a shared state can be joined even when they come from different rollouts.
    To put this into practice, we introduce \textsc{Tropic}, a training algorithm for deterministic, resettable environments with verifiable outcomes.
    %
    %
    On four agentic tasks (Sokoban, Countdown, FrozenLake, WebShop), \textsc{Tropic} outperforms the strongest on-policy baselines by up to 16 percentage points.
    Changing the algebra of reinforcement learning, not just its estimators, can thus substantially improve compositional reasoning in language models.
\end{abstract}

%% file: chapters/1_introduction_3.tex
\vspace{-3mm}
\section{Introduction}
\vspace{-2mm}
\begin{wrapfigure}{r}{0.45\textwidth}
  \vspace{-5mm}
  \centering
  \includegraphics[width=\linewidth]{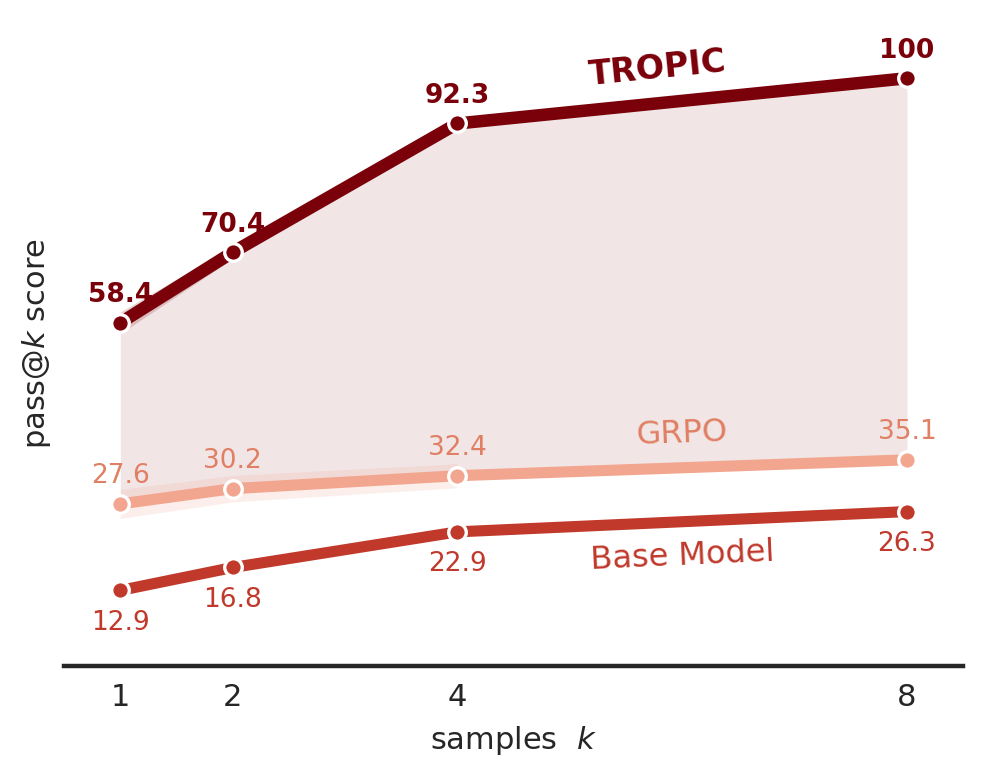}
  \vspace{-6mm}
  \caption{\emph{Sokoban pass@k results}}
  \vspace{-4mm}
  \label{fig:teaser}
\end{wrapfigure}

The best solution a language model can produce is often one it has never produced.
A model may know how to formalize a problem, derive an intermediate result, invoke the right theorem, and check an answer, and may have done each of these in earlier attempts, yet never have chained them in the order a new problem requires~\citep{press2023measuring, lake2018generalization}.
Reinforcement learning with verifiable rewards (RLVR) is the standard post-training tool for closing this gap by sampling many attempts and reinforcing the ones a verifier accepts~\citep{kimiteam2025kimi, deepseek2026v4, glm5team2026glm5, yuan2026composing, wen2026reinforcement}.
However, the classical expected-return formulation learns only from what it samples, as a trajectory is credited only after the policy has generated it end-to-end.
Recent works therefore focus on trajectory design, improving which rollouts are sampled and how they are credited~\citep{wang2026ragen2reasoningcollapseagentic, feng2025groupingrouppolicyoptimizationllm, li2026reason}.
Yet these methods keep the expected-return objective, and with it the requirement that a solution be sampled in one piece.
For composition, this requirement is the bottleneck: a solution that chains $d$ fragments, each found in a separate attempt, is sampled in one piece with probability exponentially small in $d$, even though the model has already produced every part of it.
%
%
%

%
%
%
%
%
%

Beyond failing to reward composition, expected return can also actively suppress the fragments it requires.
%
%
Suppose the model produces outcomes $A$, $B$, their composition $A\circ B$, and a failure $0$, with probabilities $(p_A,p_B,p_{AB},p_0)$. The first three receive reward one, so $J_{\mathrm{ER}}=p_A+p_B+p_{AB}$, but the composition has not yet been sampled.
Moving from $(0.30,0.30,0,0.40)$ to $(0.90,0.05,0,0.05)$ raises expected reward from $0.60$ to $0.95$.
Yet the probability that a batch of twenty rollouts contains $B$ at least once drops from $1-0.70^{20}\approx0.999$ to $1-0.95^{20}\approx0.642$.
The objective has improved, but one of the two fragments that $A\circ B$ requires is now absent from about a third of all batches.
In general, a high expected reward bounds the \emph{total} probability of currently successful behavior, but places no lower bound on any \emph{individual} fragment that a future composition requires.
This is consistent with recent findings that loss of policy entropy and reasoning diversity limits what RL post-training can achieve~\citep{cui2025entropy, jin2026revisiting}.


Both problems stem from the same choice: expected return \emph{adds up} the probabilities of alternative trajectories.
A sum does not record which trajectory succeeded, so fragments from different rollouts cannot be combined, and, because probabilities sum to one, raising one term lowers another.
Tropical RL changes this algebra. 
It still adds log-probabilities \emph{along} each trajectory, but combines alternative trajectories by their \emph{maximum}, which yields the max-plus, or \emph{tropical}, semiring.
The value of a state becomes the log-probability of its most likely verified path, and, because addition distributes over the maximum, it splits at any shared state into a best prefix and a best suffix, possibly from different rollouts.
Compositions can thus be assembled without ever being sampled in one piece, and since the maximum ignores all but the best path, the value never loses a verified path once found.
RL thus shifts from learning what the policy samples to learning what can be assembled from it.


We develop this shift from first principles to a practical algorithm. Our contributions are:
(i)~a \textsc{diagnosis}: we show that expected return, by adding probabilities over alternative trajectories, loses track of which solution succeeded and lets verified solutions be forgotten (Section~\ref{sec:flaw});
(ii)~a \textsc{forcing result}: we prove that the maximum is the only aggregation whose value always names a verified path and never decreases as new paths are found (Theorem~\ref{thm:central}); and
(iii)~\textsc{Tropic}: an algorithm of Tropical RL for agentic environments that joins fragments across rollouts and outperforms all baselines (\refig{fig:stitch}, \refig{fig:teaser}) on every task and model we test, by up to 16 points over the strongest one.

%% file: chapters/preliminaries.tex
\vspace{-3mm}
\section{Preliminaries}
\label{sec:preliminaries}
\vspace{-2mm}
\textbf{Notations.} Let \(x\) denote a reasoning problem and let a state \(s \in \mathcal{S}\)
represent a computational interface, such as a token prefix. An action \(a \in \mathcal{A}(s)\) is a token, reasoning step, tool call, or reusable computation that produces the next state $s' = T(s,a)$. Let $\bar{\pi}(a | s)$
denote a frozen language model used as a decoding or search prior. For a
trajectory $\tau = (s_0,a_0,s_1,a_1,\ldots,s_H),$ its probability under the frozen model is
$P_{\bar{\pi}}(\tau \mid s_0) = \prod_{t=0}^{H-1} \bar{\pi}(a_t | s_t).$
Consequently, $\log P_{\bar{\pi}}(\tau \mid s_0) = \sum_{t=0}^{H-1} \log \bar{\pi}(a_t \mid s_t).$ Let \(R_x(y) \in \{0,1\}\) be a terminal verifier applied to a complete response \(y\). We write $\mathcal{Y}^{+}_x(s) = \left\{ \tau : \tau \text{ begins at } s \text{ and terminates at } y \text{ with } R_x(y)=1 \right\}$ for the set of verified successful continuations from \(s\). A verified trajectory may receive a terminal quality score \(R_x(\tau)\). $\Sset_{\mathrm{T}}$ is the set of terminal states, and $R_x(s)$ is the terminal score for task $x$.

\textbf{Compositional reasoning and generalization.}
We call a solution \emph{compositional} when it is a sequence of reusable
computations, $\tau=f_1\circ\cdots\circ f_d$, in which each fragment $f_i$,
such as a reduction, a theorem application, or an arithmetic step, starts at
the state produced by the previous one. \emph{Compositional reasoning} is the
ability to produce such a sequence for a given problem.
\emph{Compositional generalization} is the ability to solve problems whose
solutions combine known fragments in a way that has not been observed before
\citep{lake2018generalization,keysers2020measuring,press2023measuring}. In this
paper, ``not observed'' has a specific meaning: every fragment of the solution
may already have appeared in some rollout, but never all of them together in a
single one. The question we study is how a learner can turn fragments found in
separate attempts into a verified solution that it was never able to sample.

\textbf{Expected-Return RL.} In a deterministic computation graph, the corresponding Bellman equations are $Q^{\bar{\pi}}(s,a) = r(s,a) + V^{\bar{\pi}}(T(s,a))$,  $V^{\bar{\pi}}(s) = \sum_{a \in \mathcal{A}(s)} \bar{\pi}(a \mid s) Q^{\bar{\pi}}(s,a).$
With binary verification on a deterministic computation graph, the classical
Bellman equation is
$\VER(s)=\sum_{a}\bar\pi(a\mid s)\,\VER\bigl(T(s,a)\bigr)$ with
$\VER(y)=R_x(y)$ at terminals. It uses two operations: probabilities of
successive steps are \emph{multiplied}, and probabilities of alternative
actions are \emph{added}. Unrolling the recursion yields
\begin{equation}
\VER(s)=\sum_{\tau\in\Yp(s)}\ \prod_{t}\bar\pi(a_t\mid s_t)
=\sum_{\tau\in\Yp(s)}\Ppi(\tau\mid s) .
\label{eq:ver}
\end{equation}
The classical value therefore has a direct operational meaning: \emph{\(V^{\bar{\pi}}(s)\) is the probability that one continuation sampled from \(\bar{\pi}\), beginning at \(s\), will eventually be accepted by the verifier.}

%% file: chapters/antropical_sections_3.tex
\begin{figure}[t!]
  \centering
  \vspace{-2em}
  \includegraphics[width=\linewidth]{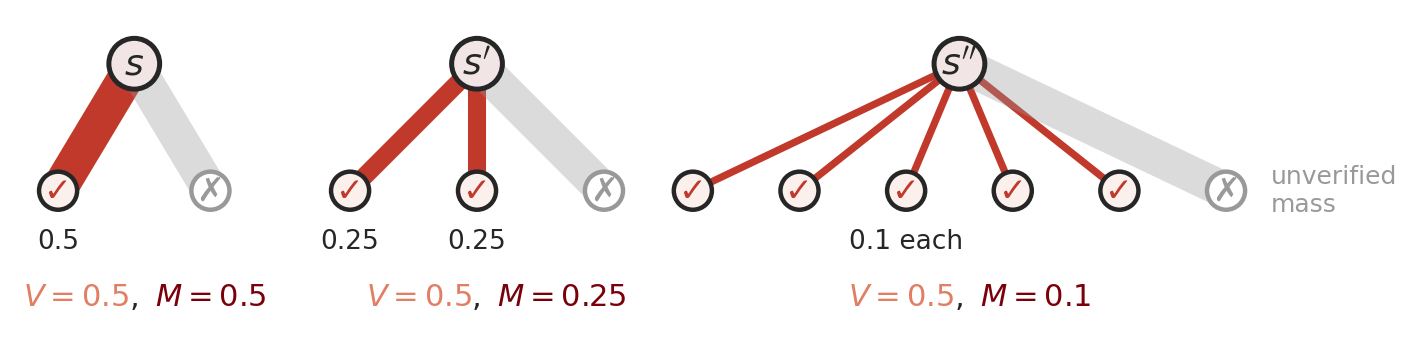} 
  \vspace{-2em}
  \caption{A difference between expected $V$ and max $M$ continuation.}
  \label{fig:Proposition_1}
  \vspace{-1em}
\end{figure}
\section{The Algebra of Expected Return}
\label{sec:flaw}
Expected return is a path sum in the semiring $(\mathbb R_{\ge0},+,\times)$ (Eq.~\ref{eq:ver}): steps along a path combine by $\otimes=\times$, and alternative paths combine by $\oplus=+$. The two operations play different roles. We keep the first and identify the second as the source of the problem.

Multiplying step probabilities is the autoregressive factorization of the model. In log space it becomes a sum of fragment scores: for an edge $f=(s,u_f,s')$ and a path $\tau$,
\begin{equation}
w(f)=\sum_{j=1}^{|u_f|}\log\bar\pi(u_{f,j}\mid x,s,u_{f,<j}),
\qquad
W(\tau)=\sum_{f\in\tau}w(f)=\log\Ppi(\tau\mid s).
\label{eq:additive}
\end{equation}
A composed path is scored from its fragments without resampling, and because multiplication distributes over addition, the value admits dynamic programming.

Adding the probabilities of alternative continuations has three properties. Each is a strength when the goal is the success probability of a single sample, and a weakness when the goal is to retain and compose computations.
\emph{(i) Addition is not selective.} $a+b$ is the score of no single path, so the value reports \emph{how much} of the actor succeeds, \textbf{not} \emph{which} computation succeeds (Observation~\ref{obs:nonsuff}).
\emph{(ii) Addition acts on normalized weights.} Since $\sum_a\pi(a\mid s)=1$, raising one alternative lowers the others (Proposition~\ref{prop:forget}).
\emph{(iii) Addition is not idempotent.} $a+a=2a$: a computation found twice counts twice, so what is sampled often is reinforced more (Appendix~\ref{app:frequency}).

\textbf{The sum does not say which computation succeeds.}
Composition needs the most accessible verified continuation, the one to reuse:
\begin{equation}
M^{\bar\pi}(s)=\max_{\tau\in\Yp(s)}\Ppi(\tau\mid s).
\label{eq:witness}
\end{equation}

\begin{observation}[Non-sufficiency]
\label{obs:nonsuff}
No function $g$ satisfies $M^{\bar\pi}(s)=g\bigl(\VER(s)\bigr)$ for all states and continuation sets.
\end{observation}
\vspace{-1em}
\begin{proof}
Let $\Yp(s)=\{\tau_1\}$ with $\Ppi(\tau_1\mid s)=0.5$, and
$\Yp(s')=\{\tau_2,\tau_3\}$ with $\Ppi(\tau_2\mid s')=\Ppi(\tau_3\mid s')=0.25$.
Then $\VER(s)=\VER(s')=0.5$, while $M^{\bar\pi}(s)=0.5$ and $M^{\bar\pi}(s')=0.25$.
\end{proof}
\vspace{-1em}
The gap can be arbitrarily large: $n$ continuations of probability $1/n$ give $\VER=1$ but $M^{\bar\pi}=1/n$. The converse also fails, since $\{0.5\}$ and $\{0.5,\,0.25\}$ share $M^{\bar\pi}=0.5$ but differ in $\VER$. Once alternatives are added, the sum cannot tell which computations produced it (Fig.~\ref{fig:Proposition_1}).

\textbf{The sum can forget a verified solution.}
We call a value \emph{policy-indexed} if it depends only on the current policy, $V_t(s)=\Phi(\pi_{\theta_t},s)$. Every on-policy expected-return method is policy-indexed, with or without baselines, critics, or entropy bonuses.

\begin{proposition}[Policy-indexed values forget]
\label{prop:forget}
Let $s$ have two actions $a,b$ with independent softmax logits. The continuation after $b$ always succeeds, $\VER(T(s,b))=1$, and the continuation after $a$ succeeds half the time, $\VER(T(s,a))=\tfrac12$. Suppose a batch samples only $a$ at $s$, some of these rollouts succeed, and the group baseline is below $1$, so $a$ receives estimated advantage $\widehat A>0$. Then after one policy-gradient step, $\pi^{+}(b\mid s)<\pi(b\mid s)$ and $\VER_{+}(s)<\VER(s)$, even though no rollout observed $b$ and the verifier's judgement of the continuation through $b$ is unchanged.
\end{proposition}
\vspace{-1em}
\begin{proof}
From $\partial\log\pi(a\mid s)/\partial z_j=\mathbf 1\{j=a\}-\pi(j\mid s)$, the step gives $\Delta z_a=\eta\widehat A\,(1-\pi(a\mid s))>0$ and $\Delta z_b=-\eta\widehat A\,\pi(b\mid s)<0$, so $\pi^{+}(b\mid s)<\pi(b\mid s)$. Only the logits at $s$ change, so $\VER(s)=\tfrac12\pi(a\mid s)+\pi(b\mid s)=\tfrac12+\tfrac12\pi(b\mid s)$ strictly decreases.
\end{proof}
\vspace{-1em}
A policy-indexed learner can restore $\pi(b\mid s)$ only by sampling $b$ again. Normalization makes reinforcement zero-sum, so a verified capability can disappear without ever being refuted.
\subsection{Replacing the aggregation}
\label{sec:forcing}
\begin{wrapfigure}{r}{0.48\textwidth}
  \vspace{-4mm}
  \centering
  \includegraphics[width=\linewidth]{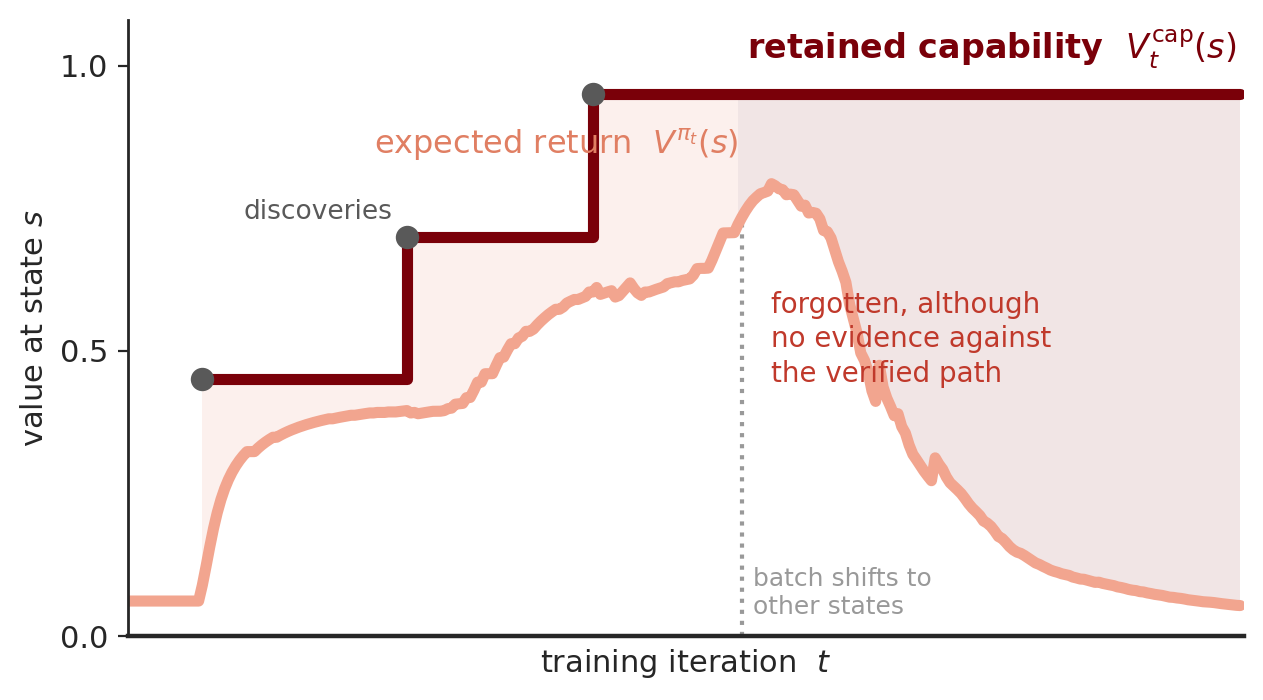}
  \vspace{-6mm}
  \caption{\emph{Lemma \ref{thm:mono} schematic illustration}}
  \vspace{-4mm}
  \label{fig:Theorem-1}
\end{wrapfigure}

The diagnosis points to a single change. We keep composition by adding
log-scores along a path (Eq.~\ref{eq:additive}), keep the model's
log-probability as the score of each path, and replace only the operation
that combines alternatives. Write $\oplus$ for that operation. The value of a
state $s$ is the aggregate of the scores of its verified continuations $V_\oplus(s)=\bigoplus_{\tau\in\Yp(s)}W(\tau)$,
where $\oplus$ is commutative and associative, so that the value does not
depend on the order in which continuations were found. Expected return is one
choice: with $\oplus=\LSE$, Eq.~\ref{eq:ver} gives $V_\oplus(s)=\log\VER(s)$.
For composition, we ask two things of the value.

\textbf{Certification.} The value of $s$ is always the score of one of
its verified continuations: $a\oplus b\in\{a,b\}$. The value then names a
concrete path that can be replayed, verified, and joined with others.

\textbf{Retention.} Finding a new verified continuation never lowers the
value: $V_\oplus(s)$ cannot decrease when a path is added to $\Yp(s)$.

\begin{theorem}[Certification and retention force the maximum]
\label{thm:central}
If $\oplus$ satisfies \textnormal{(Certification)} and \textnormal{(Retention)}, then
$a\oplus b=\max\{a,b\}$. With additive scores, the resulting structure is
$(\Rbar,\max,+)$, the max--plus or tropical semiring.
\end{theorem}
\vspace{-1em}
\begin{proof}
Adding a continuation of score $b$ to one of score $a$ gives $a\oplus b\ge a$
by (Retention); adding $a$ to $b$ gives $b\oplus a\ge b$, and $b\oplus a=a\oplus b$.
Hence $a\oplus b\ge\max\{a,b\}$. By (Certification), $a\oplus b$ equals $a$ or $b$, so
it equals $\max\{a,b\}$.
\end{proof}
\vspace{-1em}

Each requirement rules out a familiar choice. Addition and $\LSE$ satisfy
(Retention) but fail (Certification): their value exceeds every term and is the score of no
path. The minimum satisfies (Certification) but fails (Retention). Every policy-indexed value
fails (Retention) over the course of training, by Proposition~\ref{prop:forget}:
its value can fall although no verified continuation was removed.

By Theorem~\ref{thm:central}, the value of a
state is the score of its best verified continuation,
\begin{equation}
\Vtrop(s)=\max_{\tau\in\Yp(s)}W(\tau)=\log M^{\bar\pi}(s),
\label{eq:vtropdef}
\end{equation}
which is exactly the quantity that Observation~\ref{obs:nonsuff} shows
expected return cannot supply. It also has a direct meaning: it is the
distance from the actor to its nearest verified computation,
$\Vtrop(s)=-\min_{\tau\in\Yp(s)}\KL\bigl(\delta_\tau\,\|\,\Ppi(\cdot\mid s)\bigr)$,
where $\delta_\tau$ executes $\tau$ deterministically
(Lemma~\ref{lem:kl}, Appendix~\ref{app:projections}).
Section~\ref{sec:tropicalrl} builds reinforcement learning on this value.

\textbf{The maximum reverses the three failures.} It is \emph{selective}: the
value is the score of one path, recoverable by backpointers
(Observation~\ref{obs:nonsuff}). It is \emph{idempotent}:
$\max\{a,a\}=a$, so rediscovering a computation changes nothing
(Appendix~\ref{app:frequency}). And it is \emph{monotone under set growth},
which the normalized sum cannot be (Proposition~\ref{prop:forget}). Let
$\mathcal G_t=(\mathcal S_t,\mathcal E_t)$ be the graph of transitions
discovered by iteration $t$, $\Cp_t(s)$ the verified continuations
constructible from it, and
\begin{equation}
\Vcap[t](s)=\max_{\tau\in\Cp_t(s)}R_x(\tau),\qquad\max\emptyset=-\infty .
\label{eq:vcap}
\end{equation}

\begin{lemma}[Retained capability is monotone]
\label{thm:mono}
If the verifier is fixed and the memory retains its discoveries,
$\mathcal E_t\subseteq\mathcal E_{t+1}$, then $\Vcap[t+1](s)\ge\Vcap[t](s)$
for every state $s$ and iteration $t$.
\end{lemma}
\vspace{-1em}
\begin{proof}
$\mathcal E_t\subseteq\mathcal E_{t+1}$ implies $\Cp_t(s)\subseteq\Cp_{t+1}(s)$,
and the maximum is monotone under inclusion.
\end{proof}
\vspace{-1em}

Under the maximum, discovery enlarges a set, and a maximum over a growing set
cannot fall. Distributivity also survives the change,
$c+\max\{a,b\}=\max\{c+a,\,c+b\}$, so a best prefix and a best suffix can be
chosen independently at a shared state (Eq.~\ref{eq:decouple}), and dynamic
programming over exponentially many compositions remains possible.

%% file: chapters/antropical_sections_4.tex
\vspace{-1em}
\section{Tropical Reinforcement Learning}
\label{sec:tropicalrl}
\vspace{-0.5em}

Section~\ref{sec:forcing} fixed the value: the score of the best verified
continuation. This section turns it into reinforcement learning. We write its
Bellman equations (Section~\ref{sec:bellman}), prove the property that
separates it from classical RL, namely that improvement composes
(Section~\ref{sec:composes}), and build the algorithm \textsc{Tropic} on it
(Section~\ref{sec:tropic}).

\subsection{Tropical Bellman equations}
\label{sec:bellman}

With $\lambda\ge0$ weighing how much the model's likelihood counts, the
tropical Bellman equations are
\begin{equation}
\Qtrop(s,a)=\lambda\log\bar\pi(a\mid s)+r(s,a)+\Vtrop\bigl(T(s,a)\bigr),
\qquad
\Vtrop(s)=\max_{a\in\mathcal A(s)}\Qtrop(s,a).
\label{eq:tropbellman}
\end{equation}
The classical recursion averages over actions; this one takes the best. In
max--plus notation it is linear,
$\Vtrop(s)=\bigoplus_a\bigl(\lambda\log\bar\pi(a\mid s)\otimes r(s,a)\otimes\Vtrop(T(s,a))\bigr)$,
and unrolled it is the score of the best path,
\begin{equation}
\Vtrop(s)=\max_{\tau}\Bigl\{\lambda\log\Ppi(\tau\mid s)+\sum_{t}r(s_t,a_t)+R_x(s_H)\Bigr\}.
\label{eq:troppath}
\end{equation}
With a binary verifier, internal rewards are zero, accepted terminals have
value $0$ and rejected ones $-\infty$, so failed paths drop out. At
$\lambda=1$,
\begin{equation}
\Vtrop(s)=\max_{\tau\in\Yp(s)}\log\Ppi(\tau\mid s),
\label{eq:tropbinary}
\end{equation}
which is the tropical value of Eq.~\ref{eq:vtropdef}. Two special cases place
the equation among known objects. At $\lambda=0$ the model drops out, and
Eq.~\ref{eq:tropbellman} is classical Bellman optimality on a deterministic
graph; its value is the capability of Lemma~\ref{thm:mono}, which training can
never lower. At $\lambda=1$ it is Viterbi decoding under $\bar\pi$ restricted
to verified paths, and it is the zero-temperature limit of the soft Bellman
equation \citep{ziebart2010,haarnoja2018}. What is new is not the recursion
but what it is used for: a value that composes, and a learning rule built on
it.

Scores add along paths and addition distributes over max, so the best solution through a state $s$ splits into a
best prefix and a best suffix,
\begin{equation}
\max_{p\in\mathcal P_t(s),\,q\in\Cp_t(s)}W(p\circ q)
=\max_{p\in\mathcal P_t(s)}W(p)+\max_{q\in\Cp_t(s)}W(q),
\label{eq:decouple}
\end{equation}
where $\mathcal P_t(s)$ are retained prefixes ending at $s$ and $\Cp_t(s)$
retained verified suffixes starting there. The two halves may come from
different rollouts. Eq.~\ref{eq:decouple} requires that the score of a suffix
depend only on the state where it starts, not on how that state was reached.
Our approach therefore prompts the actor with the problem and the current environment state only, $\bar\pi(u\mid x,s)$, and includes the step count in $s$ so that the graph is acyclic; any joined path is also replayed and rescored before it is used. A solution of $d$ fragments, each produced with
probability about $\rho$, takes about $\rho^{-d}$ rollouts to sample in one
piece, but is available to Eq.~\ref{eq:decouple} as soon as each fragment has
been found once. Table~\ref{tab:classical-vs-tropical} lists how the change
from sum to max alters every object built on the Bellman equation.

\begin{table}[t]
\centering
\tiny
\caption{Classical and tropical RL on the same verified paths. Proofs in Appendix~\ref{app:proofs}.}
\label{tab:classical-vs-tropical}
\ccfont
\begin{tabular}{@{}p{0.25\linewidth}p{0.31\linewidth}p{0.36\linewidth}@{}}
\toprule
 & Classical (sum) & Tropical (max) \\
\midrule
The value measures & how much of the actor already succeeds & how far the actor is from its nearest verified solution (Lem.~\ref{lem:kl}) \\ \hline
Advantage & relative to the average action, must be estimated & $\le0$, exactly $0$ on the best path, computed by the DP \\\hline
Performance difference & expectation over the new policy's trajectories & a sum over one path (Thm.~\ref{thm:pdl}) \\\hline
Gradient weight on a step & visit prob.\ $\times$ action prob.\ $\times$ completion prob. & $1$ on every step of the best path (Cor.~\ref{cor:pg}) \\\hline
Rare path of depth $d$ & escape takes $\sim\rho^{-d}$ steps & $O(\log\tfrac1\rho)$ steps; MaxRL matches this on one path (Thm.~\ref{thm:plateau}) \\\hline
Training on one path & can lower an unsampled solution (Prop.~\ref{prop:forget}) & never lowers any state; shared parts gain in full (Thm.~\ref{thm:cpi-main}) \\\hline
Rediscovery & counts again & changes nothing (Thm.~\ref{thm:freqfree}) \\
\bottomrule
\end{tabular}
\vspace{-1em}
\end{table}

\subsection{Improvement composes}
\label{sec:composes}
The central difference between the two algebras shows up when the actor is trained. Under expected return, reinforcing one solution takes probability from the others, and a verified solution can lose value without ever being refuted (Proposition~\ref{prop:forget}). Under the maximum, the opposite holds.

\begin{theorem}[Compositional policy improvement]
\label{thm:cpi-main}
Let $\mathcal G$ be the retained graph of a problem $x$, and for each state
$s$ let $\mathcal C(s)$ be the verified paths in $\mathcal G$ that start at
$s$. Let the actor be a softmax policy with separate logits at each state, and
let $V_\theta(s)=\max_{\tau\in\mathcal C(s)}\log P_\theta(\tau\mid s)$. Let
$\tau^\star$ be the best verified path from the root, and take one gradient
step on its log-likelihood,
$\theta'=\theta+\eta\,\nabla_\theta\log P_\theta(\tau^\star\mid x)$, with any
$\eta>0$. Then:
\begin{enumerate}
\item[(i)] $V_{\theta'}(s)\ge V_\theta(s)$ for every state $s$ in $\mathcal G$;
\item[(ii)] if the best verified path from $s$ passes through a state $u$ of
$\tau^\star$, then $V_{\theta'}(s)-V_\theta(s)\ge V_{\theta'}(u)-V_\theta(u)>0$.
\end{enumerate}
Best paths are taken with ties broken by a fixed rule. Proof in
Appendix~\ref{app:pit}.
\end{theorem}

The state $s$ in (ii) need not lie on $\tau^\star$ and may never have been
trained from. Its best solution can be a prefix from one rollout joined at
$u$ to a suffix trained as part of another, and it inherits the full gain made
on that suffix. The reason is a property of the maximum: the step can lower
the probability of actions that are not on the best path, but no best path
uses them, so no value sees the loss. Under a sum every alternative counts,
and the loss is always seen.

The theorem is exact for an actor with separate parameters per state. With
shared parameters, as in a language model, a step on $\tau^\star$ also moves
the actor at other states, and the guarantee holds up to this interference.

\textbf{Tropical improves pass@$k$.} Let $p(x)=\sum_{\tau\in\Yp}P_\theta(\tau\mid x)$
be the pass@1 of problem $x$, so pass@$k$ is $1-(1-p(x))^k$. The best verified
path is one of the paths in this sum, so for every problem
\begin{equation}
\log p(x)\;\ge\;\Vtrop(s_0(x)),
\qquad
\text{pass@}k(x)\;\ge\;1-\bigl(1-e^{\Vtrop(s_0(x))}\bigr)^k .
\label{eq:passk}
\end{equation}
Raising the tropical value raises a lower bound on both quantities for every
problem. The bound does not by itself force pass@1 up; what it explains is
\emph{which} problems improve. The expected-return gradient on problem $x$ is
$\nabla p(x)=p(x)\,\nabla\log p(x)$, so a problem solved once in a hundred
attempts receives a hundredth of the push. The tropical gradient gives every
problem with a verified path unit weight on that path (Corollary~\ref{cor:pg}),
however rare, and the archive supplies such paths, including composed ones,
for problems the actor almost never solves alone. Pass@$k$ is decided by
exactly these problems: when $p(x)$ is small, pass@$k\approx k\,p(x)$.
Expected return spends its gradient on problems that are already solved, which
is why its pass@$k$ curve flattens toward its pass@1 (Figure~\ref{fig:teaser}).
Empirically, \textsc{Tropic} also has the highest pass@1 on every task and
model we test (Table~\ref{tab:main_results}).
\begin{takeawaybox}
\textbf{Classical RL:} improving one solution can erase another.\\
\textbf{Tropical RL:} improving one solution never lowers any other, and every
solution that shares a piece of it gains at least as much as that piece.
\end{takeawaybox}

\subsection{\textsc{Tropic}: tropical policy optimization}
\label{sec:tropic}

\textbf{Objective.} \textsc{Tropic} makes the best verified solution of each problem as likely as possible:
\begin{equation}
\Jtrop(\theta)=\mathbb{E}_{x\sim\mu}\Bigl[\max_{\tau\in\Yp(s_0(x))}\log P_\theta(\tau\mid x)\Bigr].
\label{eq:jtrop}
\end{equation}

\textbf{Memory.} For each problem, \textsc{Tropic} keeps a graph whose edges are all valid steps $f=(s,u_f,s')$ seen so far. This includes steps from failed attempts, since a later attempt may succeed from where they ended. Complete paths that are replayed from the start and accepted by the verifier form the \emph{archive}. At the start of each iteration the actor is frozen as $\bar\pi_t$, and every stored step is rescored under it, $w_t(f)=\log\bar\pi_t(u_f\mid x,s)$.

\textbf{Prefix and suffix values.} For every state, \textsc{Tropic} computes how easily the actor can \emph{finish} from it ($V_t$) and how easily it can \emph{reach} it ($F_t$):
\begin{equation}
V_t(s)=\max_{f=(s,u_f,s')}\bigl\{w_t(f)+V_t(s')\bigr\},
\qquad
F_t(s')=\max_{f=(s,u_f,s')}\bigl\{F_t(s)+w_t(f)\bigr\}.
\label{eq:dp}
\end{equation}
$V_t$ is the tropical Bellman equation (Eq.~\ref{eq:tropbellman}) run backward from verified endings, and $F_t$ is the same recursion run forward from the start. Their sum $F_t(s)+V_t(s)$ is the score of the best verified solution through $s$ (Eq.~\ref{eq:decouple}). Both passes take time linear in the size of the graph.

\textbf{Growing the memory.} New paths enter the graph in two ways.
\emph{Sampling.} Each problem receives $G$ groups of $n$ attempts. The first group starts from the initial state, and the others start from \emph{frontier} states, which the graph can reach but not yet finish from ($V_t(s)=-\infty$). To choose a start, \textsc{Tropic} first picks the depth $r^\star$ from which the fewest attempts have started so far. At that depth it picks the state that is easy to reach and rarely tried:
\begin{equation}
s^\star=\arg\max_{s\ \text{at depth}\ r^\star}\bigl\{F_t(s)-\kappa\log(1+n_s)\bigr\},
\label{eq:select}
\end{equation}
where $n_s$ is the number of earlier attempts started from $s$. Each new attempt is prepended with the best known prefix to $s^\star$, then replayed and verified as a complete path.
\emph{Composition.} Wherever a state has both a known way in and a known way out, \textsc{Tropic} joins the $L$ best prefixes with the $L$ best suffixes, then replays and verifies each joined path. Accepted joins enter the archive. This is how \textsc{Tropic} finds solutions the actor never produced in one piece.

\textbf{Training.} For each problem, \textsc{Tropic} selects up to $K$ archived paths $\basis_{x,t}$. The first is the best path $\tau_1=\tau^\star_x$. The rest are the paths with the highest coverage score $c_t(\tau)=\sum_{f\in\tau}(1+m_t(f))^{-1/2}$, where $m_t(f)$ is the number of earlier iterations in which step $f$ was trained. These extra paths rehearse rarely trained steps so they are not forgotten (Proposition~\ref{prop:forget}). The model then takes one gradient step on the per-token log-likelihood of the selected paths, averaged over the set $\mathcal{X}^{+}_t$ of problems that have at least one:
\begin{equation}
\mathcal{L}_t(\theta)=-\frac{1}{|\mathcal{X}^{+}_t|}\sum_{x\in\mathcal{X}^{+}_t}\frac{1}{|\basis_{x,t}|}\sum_{\tau\in\basis_{x,t}}\frac{1}{T_\tau}\log P_\theta(\tau\mid x),
\label{eq:loss}
\end{equation}
where $T_\tau$ is the number of tokens in $\tau$. It has no reward weights, no baseline, and no term that pushes any path down.



%% file: chapters/4_related_work.tex
\begin{figure}[t!]
  \centering
  \includegraphics[width=\linewidth]{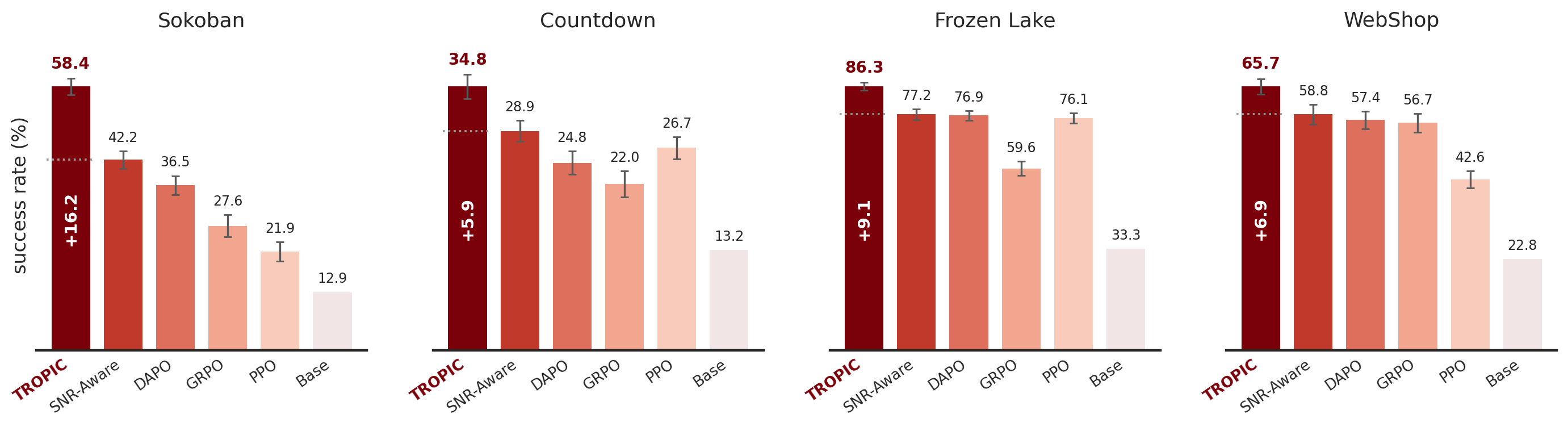}
  \vspace{-2em}
  \caption{\textsc{Tropic} performance against the baselines under the same compute budget.}
  \label{fig:stitch}
  \vspace{-1em}
\end{figure}
\vspace{-1em}
\section{Related Work}
\vspace{-1em}
Recent work shows that RL post-training can teach language models to compose skills they already have~\citep{yuan2026composing, abdulsalam2026compositional, fu2026outcome, kong2026modules} and can improve reasoning beyond the base model~\citep{wen2026reinforcement}, although its gains may stay within the base model's sampling support~\citep{yue2025does}.
Many methods improve RL training for LLMs: DAPO~\citep{yu2025dapoopensourcellmreinforcement} and SNR-aware filtering~\citep{wang2026ragen2reasoningcollapseagentic} address training collapse, GiGPO~\citep{feng2025groupingrouppolicyoptimizationllm} and VinePPO~\citep{kazemnejad2024vineppo} refine step-level credit, process supervision rewards intermediate steps~\citep{uesato2022solving, lightman2024verify}, and resets have been used for credit assignment~\citep{samanta2026credit}.
All of these works optimize expected return, so a composition must be sampled in one piece before it is reinforced. They are complementary to ours, which changes how alternative trajectories are aggregated.

Other works move beyond expected return. Several optimize pass@$k$ directly~\citep{brown2024monkeys, tang2025optimizing, walder2026pass}, and MaxRL~\citep{tajwar2026maximum} maximizes the logarithm of expected return, giving more weight to rarely solved problems.
Expected return sums the probabilities of verified paths, MaxRL takes the logarithm of this sum, and Tropical RL takes its maximum, the zero-temperature limit in which alternatives are selected rather than summed.

Closer to our method, tree search retains and extends promising partial trajectories~\citep{yao2023tree, zhang2024restmcts}, T-STAR~\citep{li2026reason} grafts corrected steps onto shared trajectory structure, and \citet{liu2026save} preserve correct prefixes when penalizing errors.
These methods reuse trajectories that were sampled; \textsc{Tropic} also composes verified solutions that were never sampled in one piece. We compare against T-STAR in Section~\ref{sec:ablation_results}.
Finally, max-plus methods have been used to compute or approximate value functions of known deterministic control problems~\citep{akian2008maxplus, mceneaney2007, berthier2020maxplus, goncalves2021maxplus, bach2019max}. We instead use the max-plus semiring to define the learning objective of a language-model policy.

%% file: chapters/results.tex
\begin{table}[t]
    \centering
    \caption{\textbf{Results (Qwen2.5-3B-Instruct \& Gemma 4 E4B-it).}
    We report Success Rate ($\uparrow$, mean $\pm$ standard deviation over three independent training seeds) and Generated Tokens per Episode ($\downarrow$) on the held-out evaluation set, together with end-to-end Training Wall-Clock Time in hours ($\downarrow$).}
    \label{tab:main_results}
    \vspace{0.06in}
    \small
    \ccfont
    \setlength{\tabcolsep}{3.8pt}
    \renewcommand{\arraystretch}{1.1}

    \resizebox{0.98\textwidth}{!}{%
    \begin{tabular}{ll ccc c ccc}
    \toprule
    \multirow{2.5}{*}{\textbf{Task}} &
    \multirow{2.5}{*}{\textbf{Method}} &
    \multicolumn{3}{c}{\textbf{Qwen2.5-3B-Instruct}} & &
    \multicolumn{3}{c}{\textbf{Gemma 4 E4B-it}} \\
    \cmidrule{3-5} \cmidrule{7-9}
    & &
    \small{Success Rate ($\uparrow$)} &
    \small{Tokens ($\downarrow$)} &
    \small{Time (h) ($\downarrow$)} & &
    \small{Success Rate ($\uparrow$)} &
    \small{Tokens ($\downarrow$)} &
    \small{Time (h) ($\downarrow$)} \\
    \midrule

    \multirow{6}{*}{\textbf{Sokoban}}
    & Base Model & 12.89               & 254 & --   & & 45.50 & 205 & -- \\
    & PPO        & \pmv{21.90}{2.14}   & 231 & 2.70 & & \pmv{71.68}{1.71} & 209 & 8.85 \\
    & GRPO       & \pmv{27.60}{2.47}   & 205 & 2.90 & & \pmv{66.53}{2.32} & 254 & 8.20 \\
    & DAPO       & \pmv{36.55}{2.11}   & 174 & 2.80 & & \pmv{72.65}{1.35} & 219 & 8.40 \\
    & SNR-Aware  & \pmv{42.20}{1.96}   & 201 & 2.10 & & \pmv{76.55}{1.94} & 237 & 8.30 \\
    \rowcolor{anthropicorange!30} \cellcolor{white}
    & TROPIC     & \pmvb{58.40}{1.83}  & 187 & 1.60 & & \pmvb{83.10}{1.88} & 224 & 8.10 \\

    \midrule

    \multirow{6}{*}{\textbf{Countdown}}
    & Base Model & 13.23               & 150 & --   & & 42.18 & 204 & -- \\
    & PPO        & \pmv{26.71}{1.46}   & 164 & 2.50 & & \pmv{66.53}{1.80} & 211 & 5.10 \\
    & GRPO       & \pmv{21.97}{1.72}   & 152 & 2.00 & & \pmv{67.65}{1.10} & 203 & 4.80 \\
    & DAPO       & \pmv{24.75}{1.51}   & 158 & 2.10 & & \pmv{63.10}{2.25} & 208 & 5.20 \\
    & SNR-Aware  & \pmv{28.93}{1.38}   & 149 & 1.75 & & \pmv{58.20}{1.95} & 201 & 4.50 \\
    \rowcolor{anthropicorange!30} \cellcolor{white}
    & TROPIC     & \pmvb{34.79}{1.62}  & 147 & 1.20 & & \pmvb{73.63}{1.48} & 186 & 4.30 \\

    \midrule

    \multirow{6}{*}{\textbf{Frozen Lake}}
    & Base Model & 33.34               & 192 & --   & & 71.30 & 218 & -- \\
    & PPO        & \pmv{76.08}{1.65}   & 110 & 2.55 & & \pmv{84.40}{1.33} & 231 & 3.70 \\
    & GRPO       & \pmv{59.57}{2.28}   & 101 & 2.20 & & \pmv{83.90}{1.85} & 290 & 3.50 \\
    & DAPO       & \pmv{76.88}{1.57}   & 106 & 2.30 & & \pmv{85.10}{2.12} & 281 & 3.90 \\
    & SNR-Aware  & \pmv{77.24}{1.74}   & 103 & 1.80 & & \pmv{79.50}{2.45} & 220 & 3.20 \\
    \rowcolor{anthropicorange!30} \cellcolor{white}
    & TROPIC     & \pmvb{86.35}{1.29}  & 100  & 1.35 & & \pmvb{92.30}{1.10} & 227 & 3.30 \\

    \midrule

    \multirow{6}{*}{\textbf{WebShop}}
    & Base Model & 22.75               & 686 & --   & & 55.27 & 521 & -- \\
    & PPO        & \pmv{42.57}{2.11}   & 413 & 9.00 & & \pmv{68.42}{1.94} & 496 & 18.20 \\
    & GRPO       & \pmv{56.68}{2.34}   & 375 & 7.90 & & \pmv{73.16}{1.81} & 472 & 16.60 \\
    & DAPO       & \pmv{57.39}{2.19}   & 390 & 8.00 & & \pmv{74.58}{1.69} & 481 & 17.10 \\
    & SNR-Aware  & \pmv{58.79}{2.51}   & 382 & 5.90 & & \pmv{76.84}{1.73} & 458 & 14.80 \\
    \rowcolor{anthropicorange!30} \cellcolor{white}
    & TROPIC     & \pmvb{65.72}{1.96}  & 366 & 4.80 & & \pmvb{81.40}{1.55} & 439 & 14.50 \\

    \bottomrule
    \end{tabular}
    }
\vspace{-1em}
\end{table}

\vspace{-0.5em}
\section{Experiments and Results}
\label{sec:experiments}
\vspace{-0.5em}
\textbf{Environments.}
We evaluate \textsc{Tropic} on four interactive tasks from RAGEN-2~\citep{wang2026ragen2reasoningcollapseagentic}:
(i)~\emph{Sokoban}~\citep{weber2018imaginationaugmentedagentsdeepreinforcement}, a grid-based planning task in which the agent pushes boxes onto targets while avoiding irreversible dead ends so that solutions depend on reusing good partial prefixes,
(ii)~\emph{Countdown}~\citep{katz2026seeminglysimpleplanningproblems}, an exact-target arithmetic task with many valid solutions in which intermediate results from different attempts can be recombined,
(iii)~\emph{FrozenLake}~\citep{brockman2016openaigym}, a grid world with terminal holes and sparse rewards in which safe partial paths must be retained and extended, and
(iv)~\emph{WebShop}~\citep{yao2023webshopscalablerealworldweb}, a long-horizon web-shopping task in which search, navigation, and attribute selection must be composed under delayed feedback.
All environments are deterministic and resettable. 
Further details are provided in \refapp{app:RL_environments}.
\\[2pt]
\textbf{Training and Baselines.}
For each task, we train Qwen-2.5-3B-Instruct~\citep{qwen2025qwen25technicalreport} and Gemma 4 E4B-it~\citep{gemmateam2026gemma4} using PPO~\citep{schulman2017proximalpolicyoptimizationalgorithms}, GRPO~\citep{shao2024deepseekmath}, DAPO~\citep{yu2025dapoopensourcellmreinforcement}, RAGEN's SNR-Aware method~\citep{wang2026ragen2reasoningcollapseagentic}, and \textsc{Tropic}.
All methods are trained under a matched experimental setup, using the same GPU, task-relevant observations, rollout budget, environment workers, and the same observation representation.
We discuss the differences in restarts, replay, and stopping in~\refapp{app:training_and_evaluation_setup}.
We evaluate all methods across three independent training seeds on a predefined held-out split of unseen instances.
We report the task success rate, end-to-end training time, and the average number of generated tokens per episode.
\\[2pt]
\textbf{Results.}
Table~\ref{tab:main_results} and \refig{fig:stitch} show that \textsc{Tropic} achieves the highest success rate on every task and for both models, while generating a comparable or smaller number of tokens and training faster in all but one setting.
The largest gains appear on Sokoban, where irreversible dead ends make complete solutions rare. 
\textsc{Tropic} improves over the strongest baseline (SNR-Aware) by $16.2$ points with Qwen and by $6.6$ points with Gemma, and trains in $1.6$ instead of $2.1$ hours with Qwen.
On Countdown, where exact verification leaves most attempts without signal, \textsc{Tropic} gains $5.9$ points over SNR-Aware with Qwen and $6.0$ points over GRPO with Gemma.
On FrozenLake, where a single misstep ends an episode, retaining safe partial paths yields $9.1$ points over SNR-Aware with Qwen and $7.2$ points over DAPO with Gemma.
On WebShop, where late mistakes hide correct search and selection steps, \textsc{Tropic} improves over SNR-Aware by $6.9$ points with Qwen and $4.6$ points with Gemma, while training fastest with both models.
Among the non-tropical baselines, no method is consistently strongest. 
SNR-Aware is the best baseline on every task with Qwen but the weakest on Countdown and FrozenLake with Gemma, and the ordering of PPO, GRPO, and DAPO changes from task to task, suggesting that refinements within the expected-return objective do not transfer reliably across tasks and models.
\textsc{Tropic}, in contrast, is best in every setting and for both model families.

\vspace{-1em}
\subsection{Ablation Results}
\vspace{-0.5em}
\label{sec:ablation_results}
\begin{wraptable}{r}{0.52\textwidth}
    \vspace{-2em}
    \centering
    \caption{\textbf{Ablations and Cross-Comparison (Sokoban, Qwen2.5-3B-Instruct).}
    Success rate, tokens per episode, training time, and \emph{First-Solve}.}
    \label{tab:sokoban_ablation}
    \vspace{0.06in}
    \ccfont
    \setlength{\tabcolsep}{3.8pt}
    \renewcommand{\arraystretch}{1.1}
    \resizebox{\linewidth}{!}{%
    \begin{tabular}{l cccc}
    \toprule
    \textbf{Method} &
    \small{Success Rate ($\uparrow$)} &
    \small{Tokens ($\downarrow$)} &
    \small{Time (h) ($\downarrow$)} &
    \small{First-Solve ($\uparrow$)} \\
    \midrule
    Base Model     & 12.89             & 254 & --   & -- \\
    \midrule
    MaxRL          & \pmv{51.30}{2.31} & 202 & 2.53 & -- \\
    GiGPO          & \pmv{47.60}{2.18} & 199 & 3.11 & -- \\
    GiGPO + T-STAR & \pmv{53.10}{2.18} & 194 & 3.80 & -- \\
    \midrule
    TROPIC (No Comp.)   & \pmv{40.80}{2.26} & 205 & 1.45 & 0.0 \\
    TROPIC (Root-Only)  & \pmv{51.80}{2.12} & 196 & 1.92 & \pmv{24.7}{3.2} \\
    TROPIC (Succ.-Only) & \pmv{53.30}{1.97} & 192 & 1.51 & 0.0 \\
    \rowcolor{anthropicorange!30}
    Full TROPIC         & \pmvb{58.40}{1.83} & 187 & 1.58 & \pmvb{31.6}{3.5} \\
    \bottomrule
    \end{tabular}
    }
    \vspace{-1em}
\end{wraptable}
To isolate the contribution of each core component in \textsc{Tropic}, we evaluate three ablations: 
(i) \emph{No Composition} retains persistent memory and frontier acquisition but disables cross-rollout graph recombination, restricting updates strictly to directly sampled paths, 
(ii) \emph{Root-Only} preserves compositional reuse but forces all rollouts to initialize at $s_0$, removing frontier continuation, and 
(iii) \emph{Successful-Only} maintains frontier exploration but restricts recombination to fragments originating from successful trajectories, ignoring useful transitions in failed rollouts. 
Additionally, we track \emph{First-Solve}, the proportion of solved instances whose first verified solution was synthesized via cross-rollout composition before any complete successful trajectory was observed. 
By definition, \emph{First-Solve} is zero for \emph{No Composition} (which lacks recombination) and \emph{Successful-Only} (which requires a pre-existing successful trajectory for fragment reuse).
\\[2pt]
\textbf{Component Ablations.}
Table~\ref{tab:sokoban_ablation} reports ablations for Qwen on Sokoban.
The results show that composition is the main source of improvement. 
Disabling it reduces success from $58.4\%$ to $40.8\%$, the largest drop and roughly the level of the strongest baseline, SNR-Aware ($42.2\%$), even though the archive and frontier restarts are kept.
Removing frontier restarts (\emph{Root-Only}) reduces success to $51.8\%$ and \emph{First-Solve} from $31.6\%$ to $24.7\%$, since every attempt must regenerate known prefixes before exploring new suffixes.
Restricting composition to successful trajectories reduces success to $53.3\%$, showing that fragments from failed rollouts remain useful building blocks.
\\[2pt]
\textbf{Comparison with Alternative Methods.}
We further ask whether stronger success-oriented objectives or cross-trajectory credit assignment can match \textsc{Tropic}.
To this end, we compare against MaxRL~\citep{tajwar2026maximum}, which replaces expected reward with a maximum-likelihood objective, GiGPO~\citep{feng2025groupingrouppolicyoptimizationllm}, which combines trajectory- and step-level advantages, and T-STAR~\citep{li2026reason}, which grafts corrected steps onto shared trajectory structure, instantiated on top of GiGPO.
They reach $51.3\%$, $47.6\%$, and $53.1\%$, respectively, while \textsc{Tropic} reaches $58.4\%$, $5.3$ points above the strongest of them, with the fewest tokens and the shortest training time.

\vspace{-1.5em}
\section{Conclusion and Discussion}
\label{sec:conclusion_and_discussion}
\vspace{-1em}
RL has moved from tabular dynamic programming to deep function approximation and now to post-training large language models. 
Each wave changed how values are represented and estimated, but all kept the same objective. 
This objective measures how often a sampled attempt succeeds. 
Compositional reasoning instead asks which verified solution can be assembled from what the model has already produced.
We showed that answering this question requires only changing the viewpoint of the algebra and herewith a single operation. 
Alternative trajectories can be combined by a maximum instead of a sum. 
The resulting tropical value names a concrete verified path, never loses a path once found, and splits at any shared state into a best prefix and a best suffix, so solutions can be assembled from fragments that were never sampled together. 
Our concrete implementation \textsc{Tropic} builds on these properties and outperforms the strongest on-policy baselines on every task and model we test.
We believe this change matters most for self-improving models. 
In particular, such a model must keep what it has verified and build new solutions from old ones. 
%
%
Transferring this growing set fully into the model's weights is the natural next step. 
Tropical RL asks what can be built from what the model has already found, and we believe this question will define the next wave.

\section*{Reproducibility statement}

In our paper we specifies the environments, and the exact definitions of the  metrics and training procedure,
\ref{app:training_and_evaluation_setup} reports the compute used. All assumptions behind the theoretical claims are
stated formally in \ref{app:proofs} and each proposition from the main section is proved there.


\subsection*{AI use statement}

In this work, we used generative AI tools to aid or polish writing, assist with code writing and with parts of the theoretical analysis.
We have not used generative AI tools for retrieval and discovery
(e.g., finding related work). All research ideas, including the proposed methods, as well as the core analysis, are original work of the authors.
We have reviewed all AI-assisted work.\\[2pt]
\textit{Text and code.}
We used OpenAI ChatGPT, Anthropic Claude, Codex, and Claude Code
to edit author-written text and assist with code development.\\[2pt]
\textit{Mathematical claims.}
We used Codex to transcribe images of our handwritten proofs
into \LaTeX. We used ChatGPT and Codex to assist with reviewing
proofs throughout the paper. With Codex, we additionally used the
Lean~4 skills package
\footnote{\url{https://github.com/cameronfreer/lean4-skills}}. 
We subsequently reviewed and refined the analysis through several iterations with AI assistance.
\\[2pt]
We take responsibility for the final content of this work,
including text, claims or artifacts produced with the aid of generative AI.

%% file: appendix/appendix_antropical.tex
%
%
%
\newcommand{\simplex}{\Delta}
\newcommand{\supp}{\operatorname{supp}}
\newcommand{\Desc}{\operatorname{Desc}}
\newcommand{\Anc}{\operatorname{Anc}}
\newcommand{\dvis}{d^{\pi}}

\section{Structural Results and Proofs}
\label{app:proofs}

This appendix collects the structural results behind
Sections~\ref{sec:flaw}--\ref{sec:tropic} and proves them in full. The main
text established three facts: expected return is not a sufficient statistic
for the best witness (observation~\ref{obs:nonsuff}), a policy-indexed value
can forget a verified computation, and the
requirements of retention force the max--plus semiring
(Proposition~\ref{prop:forcing}). Here we go further in four directions. We
show that the three objectives in our related-work table are three
\emph{distances} from the actor to the set of certified computations, and that
the tropical one is the only distance to a \emph{computation} rather than to a
distribution over computations (Section~\ref{app:projections}); this is the
result we regard as central, because it turns the choice of algebra into the
choice of what the learner is being pulled toward. We show that the tropical
learning signal is the unique one that depends only on which computations have
been found and not on how often (Section~\ref{app:frequency}). We prove a
policy improvement theorem in which training on a single certified path
improves the tropical value at every state of the retained graph, with the
improvement propagating undamped to interfaces that were never trained from,
and we prove that this improvement step is not an expected-return improvement
step and cannot be selected from expected-return information
(Section~\ref{app:pit}). Finally we prove the tropical analogue of the
performance difference lemma, derive from it the tropical policy gradient and
its lack of visitation weights, and give an exponential-versus-logarithmic
separation between expected-return and tropical gradient ascent on a chain
(Section~\ref{app:pdl}).

One idea organizes all of it. The tropical value is a functional of a
\emph{set} of verified computations, aggregated by a maximum; expected return is
a functional of a \emph{measure} over them, aggregated by a sum. A set grows
under discovery and has a most accessible element; a measure is normalized and
has only a total. Every theorem below is a consequence of one of three
properties of the maximum that the sum lacks: it is idempotent, it is attained
by one of its arguments, and it is insensitive to what happens to its
non-maximal arguments.

\subsection{Setting and conventions}
\label{app:setting}

Throughout, the environment is a finite acyclic computation graph. A
nonterminal state $s$ has a finite action set $\mathcal A(s)$ and a
deterministic transition $T(s,a)$; terminal states $y\in\mathcal S_T$ carry a
terminal score, which in the binary case is the boundary
$R_x(y)\in\{0,-\infty\}$ of Section~\ref{sec:tropicalrl}. A \emph{path} from
$s$ is a sequence $\tau=\big((s_0,a_0),\dots,(s_{H-1},a_{H-1}),y\big)$ with
$s_0=s$, $s_{t+1}=T(s_t,a_t)$ and $y=s_H\in\mathcal S_T$. Because the graph is
acyclic a path visits each state at most once; we write $S_\tau$ for its set of
nonterminal states and $a_\tau(s)$ for the action it takes at $s\in S_\tau$.
$\Desc(v)$ denotes the set of states reachable from $v$, including $v$, and
$\Anc(v)$ the set of states from which $v$ is reachable, excluding $v$;
acyclicity means $\Desc(v)\cap\Anc(v)=\emptyset$.

An actor is a policy $\pi$ with $\pi(a\mid s)>0$ for all admissible $a$. It
induces the path probability $\Ppi(\tau\mid s)=\prod_t\pi(a_t\mid s_t)$, a
probability distribution over the set $\Omega(s)$ of complete paths from $s$.
When the actor is \emph{tabular softmax} we write
$\pi_z(a\mid s)=\exp(z_{s,a})/\sum_{a'\in\mathcal A(s)}\exp(z_{s,a'})$ with
logits $z_{s,\cdot}$ that are independent parameters at each state. We use
$\sigma(z)=1/(1+e^{-z})$ for the two-action case.

A \emph{certified set} $\mathcal C(s)\subseteq\Omega(s)$ is a set of paths from
$s$ that end at terminals with $R_x=0$; it may be the full set $\Yp(s)$ or the
retained set $\Cp_t(s)$ of Section~\ref{sec:tropicalrl}. The certified set
inherits a graph: its states and edges are those visited by some certified
path, and every retained-graph quantity below is computed on it. An action of
the environment that is admissible at $s$ but leads to no certified terminal is
treated as an edge to a rejected terminal, with value $-\infty$; this lets the
softmax range over all of $\mathcal A(s)$ while the dynamic program ranges over
certified structure, and it is exactly how TROPIC's retained graph is used.

Path scores are additive with the accessibility weight $\lambda>0$ and the
internal reward $r$:
$w(s,a)=\lambda\log\pi(a\mid s)+r(s,a)$ and
$W_\pi(\tau)=\sum_{t}w(s_t,a_t)+g(y)$ with $g(y)$ the terminal score. In
Sections~\ref{app:projections}--\ref{app:pit} we take $\lambda=1$, $r\equiv0$
and the binary boundary, so that $W_\pi(\tau)=\log\Ppi(\tau\mid s)$ for
certified $\tau$; Section~\ref{app:pdl} keeps $\lambda$ and $r$ general. The
tropical values $\Qtrop,\Vtrop,\Atrop$ are those of
Equation~\eqref{eq:tropbellman}; when the actor is clear from context we drop
the superscript, and we write $V'$, $w'$ for quantities under an updated
actor. The convention $\max\emptyset=-\infty$ is in force, and any sum
containing $-\infty$ equals $-\infty$.

\subsection{Three projections onto the certified set}
\label{app:projections}

The main text compares expected return, maximum likelihood and the tropical
value by the operation each uses to aggregate verified paths: sum,
log-sum-exp, max. That comparison is algebraic. There is an equivalent
geometric comparison that is more revealing, because it says what each
objective is pulling the actor toward. Fix a state $s$, the actor's path
distribution $P=\Ppi(\cdot\mid s)$ on $\Omega(s)$, and a certified set
$\mathcal C=\mathcal C(s)$. The set of all distributions supported on
$\mathcal C$ is a simplex $\simplex(\mathcal C)$; its vertices are the point
masses $\delta_\tau$, one per certified path, and each vertex is a
\emph{deterministic policy that executes one certified computation}. The
three objectives are the actor's distance to this simplex measured three ways.

\begin{lemma}[Variational identities]
\label{lem:kl}
Let $P$ be a probability distribution on a finite set $\Omega$ with
$P(\omega)>0$ for all $\omega$, and let $\mathcal C\subseteq\Omega$ be
nonempty. Write $q^\star=P(\cdot\mid\mathcal C)$ for the conditional of $P$ on
$\mathcal C$ and $\tau^\star\in\arg\max_{\tau\in\mathcal C}P(\tau)$. Then
\begin{enumerate}
\item[(i)] $\displaystyle\min_{q\in\simplex(\mathcal C)}\KL(q\,\|\,P)=-\log P(\mathcal C)$,
attained uniquely at $q=q^\star$;
\item[(ii)] $\displaystyle\min_{\tau\in\mathcal C}\KL(\delta_\tau\,\|\,P)=-\max_{\tau\in\mathcal C}\log P(\tau)$,
attained exactly at the maximizers of $P$ on $\mathcal C$;
\item[(iii)] $\displaystyle\max_{\tau\in\mathcal C}\log P(\tau)=\log P(\mathcal C)+\log q^\star(\tau^\star)$.
\end{enumerate}
Consequently, with $P=\Ppi(\cdot\mid s)$ and $\mathcal C=\mathcal C(s)$,
\begin{equation}
\VER(s)=P(\mathcal C),\qquad
\log\VER(s)=-\min_{q\in\simplex(\mathcal C)}\KL(q\,\|\,P),\qquad
\Vtrop(s)=-\min_{\tau\in\mathcal C}\KL(\delta_\tau\,\|\,P),
\label{eq:threeproj}
\end{equation}
and the gap between the last two is
$\Vtrop(s)-\log\VER(s)=\log q^\star(\tau^\star\mid s)\le0$, the log posterior
share of the best witness among the actor's successes from $s$.
\end{lemma}

\begin{proof}
(i) For $q\in\simplex(\mathcal C)$ every term of the divergence with
$q(\tau)=0$ vanishes, so
$\KL(q\|P)=\sum_{\tau\in\mathcal C}q(\tau)\log\frac{q(\tau)}{P(\tau)}$. Since
$P(\tau)=q^\star(\tau)P(\mathcal C)$ for $\tau\in\mathcal C$,
\[
\KL(q\|P)=\sum_{\tau\in\mathcal C}q(\tau)\log\frac{q(\tau)}{q^\star(\tau)}
+\sum_{\tau\in\mathcal C}q(\tau)\log\frac{1}{P(\mathcal C)}
=\KL(q\|q^\star)-\log P(\mathcal C).
\]
The first term is nonnegative and vanishes exactly when $q=q^\star$, which
gives the minimum and its uniqueness.
(ii) $\KL(\delta_\tau\|P)=1\cdot\log\frac{1}{P(\tau)}=-\log P(\tau)$, so
minimizing over $\tau\in\mathcal C$ is maximizing $\log P(\tau)$.
(iii) $q^\star(\tau^\star)=P(\tau^\star)/P(\mathcal C)$; take logarithms.
The identifications in~\eqref{eq:threeproj} are the definitions of
$\VER$ and, by Equation~\eqref{eq:tropbinary}, of $\Vtrop$ under the binary
boundary with $\lambda=1$; the gap is (iii). Since $q^\star$ is a
probability, $\log q^\star(\tau^\star)\le0$.
\end{proof}

The lemma says that expected return measures how much of the actor's mass
already lies on certified computations; maximum likelihood measures how far the
actor is from the \emph{nearest distribution over} certified computations; the
tropical value measures how far the actor is from the \emph{nearest certified
computation}. The gap term quantifies the difference between the last two: it
is zero precisely when all of the actor's success mass at $s$ is concentrated
on one path, and it is the information Proposition~\ref{prop:nonsuff} shows
expected return cannot supply. The three are related by exponentiation and
by adding the gap, so nothing is lost in the identities; what differs is the
\emph{minimizer}, which is the object learning moves toward. For the hull
distance it is the posterior $q^\star$, a distribution; for the vertex distance
it is a path.

We now show that if one asks the value to describe a path, and asks it never to
decrease when a certified path is added, then the aggregation is forced to be
the maximum. This is a second route to Proposition~\ref{prop:forcing}, and the
two hypotheses are exactly the two operational requirements of the system in
Section~\ref{sec:tropic}: the value must refer to something that can be
replayed, composed and verified, and the archive must never lose what it has
certified.

\begin{definition}[Aggregation; certification; retention]
\label{def:aggregation}
An \emph{aggregation} is a commutative and associative binary operation
$\oplus:\mathbb R\times\mathbb R\to\mathbb R$. For a finite nonempty multiset
$A$ of reals we write $\bigoplus A$ for the iterated operation, which is
well defined by commutativity and associativity, and we set
$\bigoplus\emptyset=-\infty$. Given path scores, the induced value is
$V_\oplus(s)=\bigoplus\{W(\tau):\tau\in\mathcal C(s)\}$. We consider three
properties.
\begin{itemize}
\item[\textnormal{(Cont)}] $\oplus$ is continuous on $\mathbb R^2$.
\item[\textnormal{(Cert)}] $\oplus$ is \emph{selective}: $a\oplus b\in\{a,b\}$
for all $a,b\in\mathbb R$. By induction, $\bigoplus A\in A$ for every finite
nonempty multiset: the value of a state is the score of one of its certified
continuations.
\item[\textnormal{(Ret)}] Adding a continuation never lowers the value:
$\bigoplus(A\uplus\{c\})\ge\bigoplus A$ for every finite multiset $A$ and every
$c\in\mathbb R$.
\end{itemize}
\end{definition}

\begin{theorem}[Certification and retention force the tropical value]
\label{thm:certret}
An aggregation satisfying \textnormal{(Cont)}, \textnormal{(Cert)} and
\textnormal{(Ret)} is the maximum: $a\oplus b=\max\{a,b\}$ for all $a,b$.
Conversely the maximum satisfies all three. Consequently, with the additive
scores of Equation~\eqref{eq:additive}, the induced value is the tropical
value $V_\oplus=\Vtrop$ of Equation~\eqref{eq:tropbellman}, and by
Lemma~\ref{lem:kl} it is the negative KL distance from the actor to the
nearest certified vertex.
\end{theorem}

\begin{proof}
Let $U=\{(a,b)\in\mathbb R^2:a>b\}$, an open half-plane and therefore a
connected set. Define
\[
F_1=\{(a,b)\in U:\ a\oplus b=a\},\qquad F_2=\{(a,b)\in U:\ a\oplus b=b\}.
\]
By (Cont) the maps $(a,b)\mapsto a\oplus b-a$ and $(a,b)\mapsto a\oplus b-b$
are continuous, so $F_1$ and $F_2$ are closed subsets of $U$. By (Cert),
$F_1\cup F_2=U$. On $U$ we have $a\neq b$, so $F_1\cap F_2=\emptyset$. A
connected set cannot be written as the disjoint union of two nonempty closed
subsets, hence $F_1=U$ or $F_2=U$.

Suppose $F_1=U$. Then $a\oplus b=a=\max\{a,b\}$ whenever $a>b$. For $a<b$,
commutativity gives $a\oplus b=b\oplus a=b=\max\{a,b\}$. On the diagonal
(Cert) gives $a\oplus a=a$. So $\oplus=\max$.

Suppose instead $F_2=U$. The same argument gives $\oplus=\min$. But then
(Ret) fails: for $A=\{a\}$ and $c<a$,
$\bigoplus(A\uplus\{c\})=\min\{a,c\}=c<a=\bigoplus A$. Hence this case is
excluded and $\oplus=\max$.

Conversely $\max$ is continuous, commutative, associative, selective, and
$\max(A\uplus\{c\})\ge\max A$. With additive scores,
$V_{\max}(s)=\max_{\tau\in\mathcal C(s)}W(\tau)$, which is the path form of
Equation~\eqref{eq:tropbellman} under the binary boundary, and the last claim
is Lemma~\ref{lem:kl}(ii).
\end{proof}

The proof is short, and it is worth being clear about where the content lies.
(Cont) and commutativity are hygiene: values should vary continuously with
accessibilities and should not depend on the order in which continuations are
listed. The two substantive hypotheses each eliminate one candidate.
Certification eliminates every soft aggregation, because a soft aggregate of
two different scores lies strictly between or above them and is the score of
nothing; retention eliminates the minimum, the only other selective candidate.
What survives is the maximum, and with it the witness, idempotence, and
distributivity. The next proposition records which familiar aggregations fail
which hypothesis; the point is not that they are wrong, but that each is the
right value for a different question.

\begin{proposition}[Which aggregations fail which hypothesis]
\label{prop:whofails}
\begin{enumerate}
\item[(i)] $\LSE(a,b)=\log(e^a+e^b)$, the aggregation of maximum likelihood,
satisfies \textnormal{(Cont)}, commutativity, associativity and
\textnormal{(Ret)}, and fails \textnormal{(Cert)}.
\item[(ii)] Addition of probabilities $p\oplus q=p+q$, the aggregation of
expected return on the probability scale, satisfies \textnormal{(Ret)} and
fails \textnormal{(Cert)}.
\item[(iii)] The minimum satisfies \textnormal{(Cont)}, \textnormal{(Cert)},
commutativity and associativity, and fails \textnormal{(Ret)}.
\item[(iv)] No policy-indexed value satisfies the dynamic form of
\textnormal{(Ret)}: by Proposition~\ref{prop:forget}, such a value can
strictly decrease on a state all of whose certified continuations remain
certified.
\end{enumerate}
\end{proposition}

\begin{proof}
(i) $\LSE$ is continuous, symmetric, and associative since
$\LSE(\LSE(a,b),c)=\log(e^a+e^b+e^c)$. For $\bigoplus(A\uplus\{c\})=
\log\big(\sum_{a\in A}e^a+e^c\big)>\log\sum_{a\in A}e^a$, so (Ret) holds. For
any $a,b$, $\LSE(a,b)>\max\{a,b\}$ because $e^a+e^b>e^{\max\{a,b\}}$; hence
$\LSE(a,b)\notin\{a,b\}$ and (Cert) fails.
(ii) $p+q\ge p$, so (Ret) holds; $p+q\notin\{p,q\}$ when $p,q>0$, so (Cert)
fails.
(iii) $\min\{a,b\}\in\{a,b\}$, continuity and the algebraic identities are
immediate; the failure of (Ret) was shown in the proof of
Theorem~\ref{thm:certret}.
(iv) is Proposition~\ref{prop:forget}.
\end{proof}

The consequences of the maximum that the rest of this appendix uses are
collected next. The first is the distributivity that makes composition a local
operation; the second is the fact that the gradient of a maximum is the
gradient of its maximizer; the third is that the value can be certified by one
verifier call. Each is elementary; what matters is that each is available
\emph{because} of Theorem~\ref{thm:certret} and unavailable under any soft
aggregation.

\begin{corollary}[Consequences of the maximum]
\label{cor:maxconseq}
Let $V(s)=\max_{\tau\in\mathcal C(s)}W(\tau)$ with additive scores.
\begin{enumerate}
\item[(a)] \textnormal{(Optimal substructure.)} If $\tau=p\circ q$ attains
$V(s)$ and $q$ starts at $v$, then $q$ attains $V(v)$. Consequently, for the
set $\mathcal P(v)$ of certified prefixes from $s$ to $v$ and the set
$\mathcal C(v)$ of certified suffixes from $v$,
$\max_{p\in\mathcal P(v),\,q\in\mathcal C(v)}W(p\circ q)
=\max_{p}W(p)+\max_{q}W(q)$, which is Equation~\eqref{eq:decouple}.
\item[(b)] \textnormal{(Gradient of the maximum.)} Let the scores depend on a
parameter $\theta\in\mathbb R^n$ through continuously differentiable functions
$\theta\mapsto W_\theta(\tau)$, and suppose that at $\theta_0$ the maximizer
$\tau^\star$ over $\mathcal C(s)$ is unique. Then $\theta\mapsto V_\theta(s)$
is continuously differentiable on a neighbourhood of $\theta_0$ and
$\nabla_\theta V_{\theta_0}(s)=\nabla_\theta W_{\theta_0}(\tau^\star)$. For
$W_\theta(\tau)=\log P_\theta(\tau\mid s)$ this is the score function of the
witness, $\sum_{(s_t,a_t)\in\tau^\star}\nabla_\theta\log\pi_\theta(a_t\mid s_t)$.
\item[(c)] \textnormal{(One-call certificate.)} If $V(s)>-\infty$ there is a
path $\tau^\star$ with $W(\tau^\star)=V(s)$; replaying $\tau^\star$ and
querying the verifier once confirms that the terminal is accepted, and the
score of the replayed path equals the value. No finite number of verifier
calls on individual paths determines $\log\VER(s)$ without knowledge of the
probabilities of all paths in $\mathcal C(s)$.
\end{enumerate}
\end{corollary}

\begin{proof}
(a) Suppose $q'$ is a certified suffix from $v$ with $W(q')>W(q)$. Then
$p\circ q'$ is a certified path from $s$ and, by additivity,
$W(p\circ q')=W(p)+W(q')>W(p)+W(q)=W(\tau)=V(s)$, contradicting that $\tau$
attains the maximum. Hence $W(q)=\max_{q'}W(q')=V(v)$. For the second claim,
additivity gives $W(p\circ q)=W(p)+W(q)$ for every compatible pair, and the
maximum of a sum of independent terms is the sum of the maxima:
$\max_{p,q}[W(p)+W(q)]\le\max_pW(p)+\max_qW(q)$ termwise, and the pair of
maximizers attains the right-hand side.

(b) Write $f_\tau(\theta)=W_\theta(\tau)$ for $\tau\in\mathcal C(s)$, finitely
many continuously differentiable functions. By uniqueness,
$\varepsilon:=f_{\tau^\star}(\theta_0)-\max_{\tau\ne\tau^\star}f_\tau(\theta_0)>0$.
By continuity of finitely many functions there is a neighbourhood $N$ of
$\theta_0$ on which $|f_\tau(\theta)-f_\tau(\theta_0)|<\varepsilon/2$ for all
$\tau$; on $N$, $f_{\tau^\star}(\theta)>f_\tau(\theta)$ for every
$\tau\ne\tau^\star$, hence $V_\theta(s)=f_{\tau^\star}(\theta)$ on $N$ and
the differentiability claim follows. For log-likelihood scores,
$\nabla_\theta\log P_\theta(\tau^\star\mid s)=\sum_t\nabla_\theta\log\pi_\theta(a_t\mid s_t)$
by the autoregressive factorization~\eqref{eq:additive}.

(c) The maximum over a finite nonempty set is attained; take $\tau^\star$ a
maximizer. Its terminal is accepted because $W(\tau^\star)>-\infty$ requires
$R_x(y)=0$. For the last claim, $\log\VER(s)=\log\sum_{\tau\in\mathcal C(s)}P(\tau\mid s)$
depends on every summand, and two certified sets with the same verified paths
queried and different unqueried probabilities give different values.
\end{proof}

\subsection{Frequency-free learning signals}
\label{app:frequency}

Proposition~\ref{prop:forget} located the cause of forgetting in normalization:
probability granted to one computation is taken from another. A related cause
of collapse is \emph{frequency bias}: in on-policy learning the weight a
computation receives in the update is its sampling probability, so
computations that are found often are reinforced more, are then found more
often, and so on. One would like a learning signal that depends on
\emph{which} computations have been discovered but not on \emph{how many
times}. This section shows that this requirement is exactly idempotence of
the aggregation, hence, by Theorem~\ref{thm:certret}, exactly the maximum.

\begin{definition}[Frequency-free aggregation]
\label{def:freqfree}
Let $\oplus$ be an aggregation as in Definition~\ref{def:aggregation}. For a
finite multiset $A$ let $\supp(A)$ be the set of its distinct elements.
$\oplus$ is \emph{frequency-free} if $\bigoplus A=\bigoplus\supp(A)$ for every
finite nonempty multiset $A$. It is \emph{idempotent} if $a\oplus a=a$ for
all $a$.
\end{definition}

\begin{theorem}[Frequency-free if and only if idempotent]
\label{thm:freqfree}
An aggregation is frequency-free if and only if it is idempotent. In
particular the maximum is frequency-free, while $\LSE$ and addition are not.
\end{theorem}

\begin{proof}
Suppose $\oplus$ is frequency-free. For the multiset $A=\{a,a\}$ we have
$\supp(A)=\{a\}$, so $a\oplus a=\bigoplus A=\bigoplus\{a\}=a$: $\oplus$ is
idempotent.

Suppose $\oplus$ is idempotent. For a finite nonempty multiset $A$ let
$n(A)=|A|-|\supp(A)|\ge0$ count its repetitions. We show
$\bigoplus A=\bigoplus\supp(A)$ by induction on $n(A)$. If $n(A)=0$ then $A$
is a set and there is nothing to prove. If $n(A)\ge1$, some element $a$
occurs at least twice, and we may write $A=\{a,a\}\uplus A'$ for a multiset
$A'$. If $A'=\emptyset$ then $\bigoplus A=a\oplus a=a=\bigoplus\{a\}=\bigoplus\supp(A)$.
Otherwise, using commutativity and associativity to group the two copies,
\[
\bigoplus A=(a\oplus a)\oplus\bigoplus A'=a\oplus\bigoplus A'=\bigoplus(\{a\}\uplus A').
\]
The multiset $\{a\}\uplus A'$ has the same support as $A$ and
$n(\{a\}\uplus A')=n(A)-1$, so by the induction hypothesis its aggregate is
$\bigoplus\supp(A)$.

Finally $\max\{a,a\}=a$, whereas $\LSE(a,a)=a+\log2\ne a$ and $a+a=2a\ne a$
for $a\ne0$.
\end{proof}

The theorem is a statement about values. Its bearing on learning is through the
gradient: if the value is a function of the support only, so is its gradient
wherever it exists. The next proposition computes the gradients of the three
objectives on a multiset of certified paths and makes the contrast concrete.
It is the appendix form of Equation~\eqref{eq:gradfam}.

\begin{proposition}[The weights of the three objectives]
\label{prop:weights}
Let $A$ be a finite multiset of certified paths from $s_0(x)$ with scores
$W_\theta(\tau)=\log P_\theta(\tau\mid x)$, continuously differentiable in
$\theta$. Define
\[
\JER^{A}(\theta)=\sum_{\tau\in A}e^{W_\theta(\tau)},\qquad
\JML^{A}(\theta)=\log\sum_{\tau\in A}e^{W_\theta(\tau)},\qquad
\Jtrop^{A}(\theta)=\max_{\tau\in A}W_\theta(\tau),
\]
where sums run over the multiset with multiplicity. Then each gradient has
the form $\sum_{\tau\in A}\omega(\tau)\nabla_\theta\log P_\theta(\tau\mid x)$
with
\[
\omega_{\mathrm{ER}}(\tau)=P_\theta(\tau\mid x),\qquad
\omega_{\mathrm{ML}}(\tau)=\frac{P_\theta(\tau\mid x)}{\sum_{\tau'\in A}P_\theta(\tau'\mid x)},\qquad
\omega_{\mathrm{trop}}(\tau)=\mathbf 1\{\tau=\tau^\star\},
\]
the last wherever the maximizer $\tau^\star$ is unique. If a path $\tau_0\in A$
is duplicated, $\nabla\JER^{A}$ acquires an additional term
$P_\theta(\tau_0\mid x)\nabla\log P_\theta(\tau_0\mid x)$, every weight of
$\nabla\JML^{A}$ changes, and $\nabla\Jtrop^{A}$ is unchanged.
\end{proposition}

\begin{proof}
For the first, $\nabla e^{W_\theta(\tau)}=e^{W_\theta(\tau)}\nabla W_\theta(\tau)
=P_\theta(\tau\mid x)\nabla\log P_\theta(\tau\mid x)$; sum over the multiset.
For the second, $\nabla\log\sum_\tau e^{W_\theta(\tau)}=
\big(\sum_\tau e^{W_\theta(\tau)}\nabla W_\theta(\tau)\big)/\sum_{\tau'}e^{W_\theta(\tau')}$,
which is the stated posterior weighting. The third is
Corollary~\ref{cor:maxconseq}(b). Under duplication of $\tau_0$: the sum
defining $\JER^{A}$ gains a copy of the term $e^{W_\theta(\tau_0)}$ and its
gradient gains a copy of that term's gradient; the denominator of
$\omega_{\mathrm{ML}}$ increases by $P_\theta(\tau_0\mid x)$ and the numerator
of the weight on $\tau_0$ doubles, so all weights change; the multiset
maximum and its unique maximizer are unchanged, hence so is the gradient of
$\Jtrop^{A}$.
\end{proof}

The proposition gives the three objectives a common form and isolates their
difference in the weight placed on each path. The expected-return weight is the
path's own probability, which is what makes on-policy learning frequency
weighted by construction. The maximum-likelihood weight is a posterior, which
removes the dependence on absolute probability but still counts each retained
realization. The tropical weight is an indicator, and it is the only one of the
three that Theorem~\ref{thm:freqfree} allows. In TROPIC this is why the
acquisition process may be arbitrarily off-policy, why no importance weights
appear, and why the coverage counter $m_t(f)$ in
Section~\ref{sec:tropic} counts rehearsals chosen by the algorithm rather than
samples produced by the actor: the only frequency in the method is one the
learner controls.

\subsection{Compositional policy improvement}
\label{app:pit}

The classical policy improvement theorem says that acting greedily with
respect to an exact $Q^\pi$ produces a policy at least as good as $\pi$ at
every state. It requires $Q^\pi(s,a)$ for every action at every state visited,
which for expected return means the total success mass of every subtree. This
section proves a tropical analogue with a different shape. A single certified
path is trained by maximum likelihood, exactly the $K=1$ update of
Equation~\eqref{eq:loss}, and the tropical value is shown to be non-decreasing
at \emph{every} state of the retained graph, including states that were never
trained from and never visited by the trained path. Moreover the improvement
is additive over shared fragments and reaches an interface undamped: a state
whose best certified continuation merges into the trained path inherits the
full gain accumulated on that path from the merge point on.

Three properties of the maximum carry the proof, and it is worth naming them
before the formalities. \emph{Substructure}: a best path that reaches a state
on the trained path must continue along it, because suffixes of maximizers are
maximizers. \emph{Join}: the value dominates the score of any particular path,
so lower-bounding the value only requires exhibiting one path whose score did
not fall. \emph{Blindness}: the update can lower the log-probabilities of
non-greedy actions at trained states, but no best path uses any of them, so
the value never sees the loss. The third is where the algebra bites; it is
dominance (R2 of Section~\ref{sec:flaw}) in dynamic form, and it is false for
any aggregation that counts non-maximal alternatives.

Throughout this section the actor is tabular softmax with $\lambda=1$, the
boundary is binary, internal rewards are zero, and every nonterminal state on
the trained path has at least two admissible actions (a state with one action
has nothing to learn). Values are computed on the certified graph as described
in Section~\ref{app:setting}.

\begin{definition}[Greedy action and greedy path]
\label{def:greedy}
Fix a total order on the actions at every state. Given logits $z$, hence
$\pi=\pi_z$, scores $w(s,a)=\log\pi(a\mid s)$ and values
$Q(s,a)=w(s,a)+V(T(s,a))$, $V(s)=\max_aQ(s,a)$, the \emph{greedy action} at a
nonterminal $s$ with $V(s)>-\infty$ is the first maximizer
$a^\star(s)=\min\{a\in\mathcal A(s):Q(s,a)=V(s)\}$, and the \emph{greedy path}
$\tau^\star(s)$ is defined recursively by
$\tau^\star(s)=(s,a^\star(s))\circ\tau^\star(T(s,a^\star(s)))$, with
$\tau^\star(y)=y$ at terminals. The map $s\mapsto a^\star(s)$ is a
deterministic policy on the graph and $\tau^\star(s)$ is its rollout from $s$.
\end{definition}

\begin{lemma}[The greedy path attains the value and has substructure]
\label{lem:greedy}
For every state $s$ with $V(s)>-\infty$: the greedy path $\tau^\star(s)$ ends
at a terminal $y$ with $R_x(y)=0$ and $W(\tau^\star(s))=V(s)$. For every state
$u\in S_{\tau^\star(s)}$ the action taken there is $a^\star(u)$ and the suffix
of $\tau^\star(s)$ from $u$ is $\tau^\star(u)$. In particular, if two greedy
paths visit a common state they take the same action there and coincide from
that state on.
\end{lemma}

\begin{proof}
Induct in reverse topological order. At a terminal $V(y)=R_x(y)$, and
$V(y)>-\infty$ forces $R_x(y)=0=W(y)$. At a nonterminal $s$ with
$V(s)>-\infty$, the maximum $V(s)=\max_aQ(s,a)$ is attained by $a^\star(s)$,
so $V(T(s,a^\star(s)))>-\infty$ and by induction
$W(\tau^\star(T(s,a^\star(s))))=V(T(s,a^\star(s)))$ with an accepted
terminal. Then $W(\tau^\star(s))=w(s,a^\star(s))+W(\tau^\star(T(s,a^\star(s))))
=Q(s,a^\star(s))=V(s)$. The substructure statements are immediate from the
recursive definition: at each state visited, the path takes $a^\star$ of that
state and continues with the greedy path of the successor, so the suffix from
any visited $u$ is by construction $\tau^\star(u)$, which depends on $u$
alone.
\end{proof}

\begin{definition}[Tropical greedy update]
\label{def:update}
Let $\tau=\tau^\star(s_0)$ with nonterminal state set $S_\tau$, and let
$\eta>0$. The tropical greedy update is one gradient step on the
log-likelihood of $\tau$ with respect to the logits,
\[
z'=z+\eta\,\nabla_z\sum_{(s,a)\in\tau}\log\pi_z(a\mid s).
\]
Because a path visits each state once and logits at different states are
independent, this reads $z'_{s,\cdot}=z_{s,\cdot}$ for $s\notin S_\tau$ and
\[
z'_{s,a}=z_{s,a}+\eta\big(\mathbf 1\{a=a^\star(s)\}-\pi(a\mid s)\big)\qquad\text{for }s\in S_\tau,\ a\in\mathcal A(s).
\]
It is the $K=1$ update of Equation~\eqref{eq:loss} for one problem, without the
token normalization. Primes denote quantities under $z'$, and
$\Delta w(s,a)=w'(s,a)-w(s,a)$.
\end{definition}

\begin{lemma}[Local effect of the update]
\label{lem:local}
\begin{enumerate}
\item[(a)] $\Delta w(s,a)=0$ for every $s\notin S_\tau$ and every $a$.
\item[(b)] For $s\in S_\tau$ and $a\ne a^\star(s)$,
$\Delta w(s,a^\star(s))-\Delta w(s,a)=\eta\big(1-\pi(a^\star(s)\mid s)+\pi(a\mid s)\big)>0$.
\item[(c)] For $s\in S_\tau$, $\Delta w(s,a^\star(s))>0$.
\item[(d)] Consequently $\Delta w(s,a)\le\Delta w(s,a^\star(s))$ for all $s\in S_\tau$
and all $a$, with $\Delta w(s,a^\star(s))\ge0$ for all $s$ on or off $\tau$.
\end{enumerate}
\end{lemma}

\begin{proof}
(a) The logits at $s$ are unchanged, hence so is $\pi(\cdot\mid s)$.
(b) At any state, $w(s,a)=z_{s,a}-\log\sum_{a'}e^{z_{s,a'}}$, so
$\Delta w(s,a)=\Delta z_{s,a}-\big[\log\sum_{a'}e^{z'_{s,a'}}-\log\sum_{a'}e^{z_{s,a'}}\big]$,
and the bracket does not depend on $a$. Subtracting the expressions for
$a^\star(s)$ and $a$,
$\Delta w(s,a^\star(s))-\Delta w(s,a)=\Delta z_{s,a^\star(s)}-\Delta z_{s,a}
=\eta(1-\pi(a^\star(s)\mid s))+\eta\pi(a\mid s)$, which is positive since
$\pi(a^\star(s)\mid s)<1$ when $s$ has at least two actions.
(c) Write $\pi(a^\star\mid s)=\big(1+\sum_{a\ne a^\star}e^{z_{s,a}-z_{s,a^\star}}\big)^{-1}$
and likewise for $z'$. By (b) each exponent $z'_{s,a}-z'_{s,a^\star}$ is
strictly smaller than $z_{s,a}-z_{s,a^\star}$, so the sum strictly decreases
and $\pi'(a^\star\mid s)>\pi(a^\star\mid s)$, i.e.\ $\Delta w(s,a^\star(s))>0$.
(d) is (b), (c) and (a) together.
\end{proof}

Part (b) is the exact statement of blindness: whatever normalization does to
the probabilities of the other actions at a trained state, the greedy action
gains strictly more than any of them, in log units, by an amount that does not
depend on the value function at all. Part (c) says the greedy action's own
log-probability rises. Nothing is claimed about the other actions
individually; with three or more actions some of them may gain probability
from the dominant one. The proofs below never need them to lose.

\begin{lemma}[Gain bound]
\label{lem:gainbound}
For every path $\tau'$ in the certified graph,
\[
W'(\tau')-W(\tau')\ \le\ \sum_{x\in S_{\tau'}\cap S_\tau}\Delta w(x,a^\star(x)),
\]
and for every state $v$,
\[
V'(v)\ \le\ V(v)+\sum_{x\in S_\tau\cap\Desc(v)}\Delta w(x,a^\star(x)).
\]
\end{lemma}

\begin{proof}
$W'(\tau')-W(\tau')=\sum_{(x,a)\in\tau'}\Delta w(x,a)$. For $x\notin S_\tau$
the term is $0$ by Lemma~\ref{lem:local}(a); for $x\in S_\tau$ it is at most
$\Delta w(x,a^\star(x))$ by Lemma~\ref{lem:local}(d). This gives the first
inequality. For the second, if $V(v)=-\infty$ then no certified path leaves
$v$ and $V'(v)=-\infty$ as well, since the graph and the boundary are
unchanged. Otherwise every path $\tau'$ from $v$ has $S_{\tau'}\subseteq\Desc(v)$,
and since each $\Delta w(x,a^\star(x))\ge0$,
\[
V'(v)=\max_{\tau'}W'(\tau')\le\max_{\tau'}\Big[W(\tau')+\sum_{x\in S_\tau\cap\Desc(v)}\Delta w(x,a^\star(x))\Big]
=V(v)+\sum_{x\in S_\tau\cap\Desc(v)}\Delta w(x,a^\star(x)).
\]
\end{proof}

\begin{lemma}[Stability along the trained path, with exact gain]
\label{lem:stability}
Write $\tau=\big((u_0,a_0),\dots,(u_{m-1},a_{m-1}),u_m\big)$ with $u_0=s_0$,
$a_k=a^\star(u_k)$ and $u_m$ an accepted terminal, and set
$\Delta_{\ge k}=\sum_{j=k}^{m-1}\Delta w(u_j,a_j)$ for $0\le k\le m$, so
$\Delta_{\ge m}=0$. Then for every $k$:
\[
V'(u_k)=V(u_k)+\Delta_{\ge k},
\]
and for $k<m$ the action $a_k$ is the unique maximizer of $Q'(u_k,\cdot)$. In
particular the greedy policy is unchanged on $S_\tau$, the greedy path from
$s_0$ under $z'$ is again $\tau$, and $\Delta_{\ge k}>0$ for $k<m$.
\end{lemma}

\begin{proof}
Backward induction on $k$. For $k=m$, $V'(u_m)=R_x(u_m)=0=V(u_m)$ and
$\Delta_{\ge m}=0$. Assume the claim for $k+1$ and consider $u_k$.

For the greedy action, since $a_k$ attains $V(u_k)$ we have
$V(u_k)=w(u_k,a_k)+V(u_{k+1})$, and therefore
\[
Q'(u_k,a_k)=w'(u_k,a_k)+V'(u_{k+1})
=w(u_k,a_k)+\Delta w(u_k,a_k)+V(u_{k+1})+\Delta_{\ge k+1}
=V(u_k)+\Delta_{\ge k}.
\]

For any other action $a\ne a_k$, let $v=T(u_k,a)$. The states
$u_0,\dots,u_k$ are ancestors of $v$: $u_k$ is a parent of $v$, and each
$u_j$ with $j<k$ is an ancestor of $u_k$. By acyclicity none of them lies in
$\Desc(v)$, so $S_\tau\cap\Desc(v)\subseteq\{u_{k+1},\dots,u_{m-1}\}$ and
Lemma~\ref{lem:gainbound} gives $V'(v)\le V(v)+\Delta_{\ge k+1}$. Hence, using
$Q(u_k,a)=w(u_k,a)+V(v)\le V(u_k)$,
\[
Q'(u_k,a)=w'(u_k,a)+V'(v)\le w(u_k,a)+\Delta w(u_k,a)+V(v)+\Delta_{\ge k+1}
\le V(u_k)+\Delta w(u_k,a)+\Delta_{\ge k+1}.
\]
Subtracting,
\[
Q'(u_k,a_k)-Q'(u_k,a)\ \ge\ \Delta w(u_k,a_k)-\Delta w(u_k,a)
=\eta\big(1-\pi(a_k\mid u_k)+\pi(a\mid u_k)\big)>0
\]
by Lemma~\ref{lem:local}(b). (If $V(v)=-\infty$ the inequality is trivial.)
So $a_k$ is the unique maximizer of $Q'(u_k,\cdot)$ and
$V'(u_k)=Q'(u_k,a_k)=V(u_k)+\Delta_{\ge k}$. Since the first maximizer is the
unique maximizer, $a'^{\star}(u_k)=a_k$; applying this along $\tau$ shows the
greedy path from $s_0$ is unchanged. Finally $\Delta_{\ge k}>0$ for $k<m$
because each summand is positive by Lemma~\ref{lem:local}(c).
\end{proof}

\begin{theorem}[Compositional policy improvement]
\label{thm:cpi}
Let $\tau=\tau^\star(s_0)$ and let $z'$ be the tropical greedy update along
$\tau$ with any step $\eta>0$.
\begin{enumerate}
\item[(i)] \textnormal{(Pointwise improvement from one path.)} For every state
$s$ of the certified graph, $V'(s)\ge V(s)$.
\item[(ii)] \textnormal{(Additivity over shared fragments.)} For every $s$ with
$V(s)>-\infty$,
\[
V'(s)-V(s)\ \ge\ \sum_{x\in S_{\tau^\star(s)}\cap S_\tau}\Delta w(x,a^\star(x)).
\]
\item[(iii)] \textnormal{(Undamped propagation to interfaces.)} If
$\tau^\star(s)$ visits a state of $\tau$ and $u_k$ is the first such state,
then $\tau^\star(s)$ coincides with $\tau$ from $u_k$ on and
\[
V'(s)-V(s)\ \ge\ V'(u_k)-V(u_k)=\Delta_{\ge k}>0.
\]
\item[(iv)] \textnormal{(Iteration.)} Repeating the update produces, at every
state, a non-decreasing sequence of values, with the same trained path at
every step.
\end{enumerate}
\end{theorem}

\begin{proof}
(i)--(ii). If $V(s)=-\infty$ then $V'(s)=-\infty$ and (i) holds with equality.
Otherwise let $\tau^\star(s)$ be the greedy path from $s$ under $z$; by
Lemma~\ref{lem:greedy} it is certified and $W(\tau^\star(s))=V(s)$. By the
join property $V'(s)\ge W'(\tau^\star(s))$, and
$W'(\tau^\star(s))-W(\tau^\star(s))=\sum_{(x,a)\in\tau^\star(s)}\Delta w(x,a)$.
For $x\notin S_\tau$ the term is $0$ (Lemma~\ref{lem:local}(a)). For
$x\in S_\tau$, Lemma~\ref{lem:greedy} says the action taken by
$\tau^\star(s)$ at $x$ is $a^\star(x)$, so the term is
$\Delta w(x,a^\star(x))>0$ (Lemma~\ref{lem:local}(c)). Hence
$V'(s)\ge V(s)+\sum_{x\in S_{\tau^\star(s)}\cap S_\tau}\Delta w(x,a^\star(x))\ge V(s)$,
which is (ii) and (i).

(iii). The states of $\tau^\star(s)$ before $u_k$ lie outside $S_\tau$, by
the choice of $u_k$. From $u_k$ on, $\tau^\star(s)$ equals $\tau^\star(u_k)$
by Lemma~\ref{lem:greedy}, and so does the suffix of $\tau$ from $u_k$, since
$\tau=\tau^\star(s_0)$ visits $u_k$. Therefore
$S_{\tau^\star(s)}\cap S_\tau=\{u_k,\dots,u_{m-1}\}$ and the sum in (ii)
equals $\Delta_{\ge k}$, which is $V'(u_k)-V(u_k)$ by
Lemma~\ref{lem:stability} and is positive.

(iv). By Lemma~\ref{lem:stability} the greedy path from $s_0$ under $z'$ is
again $\tau$, so the next update is a tropical greedy update along the same
path and (i) applies again.
\end{proof}

Read (iii) as a statement about composition. The state $s$ was never trained
from and the path $\tau^\star(s)$ was never sampled; it is a prefix from $s$ to
$u_k$, discovered in some rollout, composed at the interface $u_k$ with a
suffix that was trained as part of a different path. The theorem says that all
of the training on that suffix is inherited by $s$, in full, in log units.
Under expected return the change at $s$ from training at $u_k$ is scaled by the
share of $s$'s success mass that flows through $u_k$, and is further reduced
by the mass lost on suppressed alternatives; under the maximum neither
attenuation exists, because the value at $s$ is the score of one path and
that path does not use any suppressed alternative.

The next proposition establishes that the improvement step of
Theorem~\ref{thm:cpi} is genuinely a different step from classical policy
improvement, in two senses: the action it reinforces cannot be determined from
expected-return information, and acting on expected-return information can
destroy the tropical value while the tropical step can destroy expected
return. The example is the one from Proposition~\ref{prop:nonsuff}, lifted
from values to decisions.

\begin{proposition}[Incompatibility with expected-return improvement]
\label{prop:incompat}
\begin{enumerate}
\item[(a)] \textnormal{(Non-identifiability.)} There exist two instances that
agree on a state $u$, on $\pi(\cdot\mid u)$, on $\VER(T(u,a))$ for every
action, and on $\VER(u)$, yet have different tropical greedy actions at $u$.
Hence no rule that selects an action from these quantities selects the
tropical greedy action on all instances.
\item[(b)] \textnormal{(Mutual harm.)} Let $u$ have two actions $a,a'$ with
$a$ the unique tropical greedy action and $a'$ the unique expected-return
greedy action, $\VER(T(u,a'))>\VER(T(u,a))$. Any update that increases
$\pi(a\mid u)$ and changes no other state strictly decreases $\VER(u)$. Any
update that increases $\pi(a'\mid u)$ and changes no other state strictly
decreases $\Vtrop(u)$ provided
$\log\frac{\pi'(a'\mid u)}{\pi(a'\mid u)}<-\Atrop(u,a')$, in particular for
all sufficiently small steps.
\end{enumerate}
\end{proposition}

\begin{proof}
(a) Take $\pi(a\mid u)=\pi(a'\mid u)=\tfrac12$ in both instances. In the first,
$T(u,a)$ has a single certified continuation of probability $\tfrac12$ and
$T(u,a')$ has three of probability $\tfrac14$ each. In the second, $T(u,a)$
has two of probability $\tfrac14$ and $T(u,a')$ one of probability
$\tfrac34$. In both, $\VER(T(u,a))=\tfrac12$, $\VER(T(u,a'))=\tfrac34$ and
$\VER(u)=\tfrac12\cdot\tfrac12+\tfrac12\cdot\tfrac34=\tfrac58$. The tropical
$Q$-values are $\log\pi(a\mid u)+\max\log P$ over the subtree: in the first
instance $\log(\tfrac12\cdot\tfrac12)=\log\tfrac14$ for $a$ against
$\log(\tfrac12\cdot\tfrac14)=\log\tfrac18$ for $a'$, so the tropical greedy
action is $a$; in the second, $\log\tfrac18$ for $a$ against
$\log(\tfrac12\cdot\tfrac34)=\log\tfrac38$ for $a'$, so it is $a'$. A rule
that is a function of the listed quantities returns the same action in both
instances and is therefore wrong in one of them.

(b) With two actions, $\pi(a'\mid u)=1-\pi(a\mid u)$ and downstream values are
unchanged by an update at $u$, so
\[
\VER(u)=\pi(a\mid u)\VER(T(u,a))+\big(1-\pi(a\mid u)\big)\VER(T(u,a'))
\]
\[
=\VER(T(u,a'))-\pi(a\mid u)\big[\VER(T(u,a'))-\VER(T(u,a))\big].
\]
The bracket is positive by hypothesis, so increasing $\pi(a\mid u)$ strictly
decreases $\VER(u)$.

For the second claim write $A=w(u,a)+\Vtrop(T(u,a))$ and
$B=w(u,a')+\Vtrop(T(u,a'))$, so $\Vtrop(u)=\max\{A,B\}=A>B$ and
$\Atrop(u,a')=B-A<0$. An update increasing $\pi(a'\mid u)$ decreases
$\pi(a\mid u)$, so $w'(u,a)<w(u,a)$ and $A'<A$, while
$B'=B+\log\frac{\pi'(a'\mid u)}{\pi(a'\mid u)}$. Then
$\Vtrop(u)'=\max\{A',B'\}<A=\Vtrop(u)$ as soon as $B'<A$, i.e.\ as soon as
$\log\frac{\pi'(a'\mid u)}{\pi(a'\mid u)}<A-B=-\Atrop(u,a')$. Since the left
side tends to $0$ with the step size and the right side is a fixed positive
number, the condition holds for all sufficiently small steps.
\end{proof}

\begin{remark}[On the coverage path]
Theorem~\ref{thm:cpi} concerns the $K=1$ update. The coverage path $\tau_2$ of
Section~\ref{sec:tropic} may take a non-greedy action at a state shared with
$\tau_1$, in which case Lemma~\ref{lem:stability} need not hold at that state
and strict monotonicity may fail there. This is a deliberate trade: the second
path is chosen for fragment coverage, not accessibility, and the price is a
possible local decrease in the best-witness accessibility in exchange for
rehearsing fragments that would otherwise erode. The theorem tells us exactly
what is given up and where.
\end{remark}

\subsection{The tropical performance difference lemma}
\label{app:pdl}

The performance difference lemma expresses the change in value between two
policies as an expectation, under the new policy's trajectory distribution, of
the old policy's advantage. Everything that is hard in classical policy
optimization lives in that expectation: it is over a distribution one does not
have, its gradient carries the visitation weight $\dvis(s)\pi(a\mid s)$ that
makes rare good paths invisible, and convergence rates inherit a distribution
mismatch coefficient from it. This section derives the tropical analogue and
shows that the expectation is replaced by a sum over a single path. In this
section $\lambda>0$ and the internal reward $r$ are general and terminal
scores $g(y)$ may be graded; the binary boundary is the case $g=R_x$.

\begin{lemma}[Path form and path identity]
\label{lem:pathid}
For an actor $\pi$ let $Q^\pi,V^\pi,A^\pi$ be given by
Equation~\eqref{eq:tropbellman} with $V^\pi(y)=g(y)$ at terminals, and let
$W_\pi(\tau)=\sum_t\big[\lambda\log\pi(a_t\mid s_t)+r(s_t,a_t)\big]+g(y)$.
\begin{enumerate}
\item[(a)] $V^\pi(s)=\max_{\tau\text{ from }s}W_\pi(\tau)$ for every $s$.
\item[(b)] For every path $\tau'$ from $s$ to a terminal $y$ with
$g(y)>-\infty$,
\[
W_\pi(\tau')=V^\pi(s)+\sum_{(s_t,a_t)\in\tau'}A^\pi(s_t,a_t).
\]
\end{enumerate}
\end{lemma}

\begin{proof}
(a) Reverse topological induction. At terminals both sides equal $g(y)$. At a
nonterminal $s$, every path from $s$ is $(s,a)$ followed by a path from
$T(s,a)$, and additivity gives
$\max_\tau W_\pi(\tau)=\max_a\big[\lambda\log\pi(a\mid s)+r(s,a)+\max_{\tau''\text{ from }T(s,a)}W_\pi(\tau'')\big]
=\max_aQ^\pi(s,a)=V^\pi(s)$ by the induction hypothesis.

(b) Along $\tau'$ every value $V^\pi(s_t)$ is finite, since the suffix of
$\tau'$ from $s_t$ is a path with finite score and (a) gives
$V^\pi(s_t)\ge$ that score. Then, with $s_H=y$,
\[
\sum_{t=0}^{H-1}A^\pi(s_t,a_t)
=\sum_{t=0}^{H-1}\big[\lambda\log\pi(a_t\mid s_t)+r(s_t,a_t)+V^\pi(s_{t+1})-V^\pi(s_t)\big]
\]
\[
=\big[W_\pi(\tau')-g(y)\big]+V^\pi(y)-V^\pi(s),
\]
by telescoping the value terms. Since $V^\pi(y)=g(y)$, this equals
$W_\pi(\tau')-V^\pi(s)$.
\end{proof}

Part (b) has no classical counterpart. It says that the deficit of \emph{any}
certified path relative to the value is exactly the sum of tropical
advantages along it: a deterministic accounting, path by path, with no
expectation. Every path has non-positive deficit because every tropical
advantage is non-positive, and the witness has deficit zero because every
action on it is a maximizer. The lemma is the reason the tropical advantage
never needs to be estimated: it is a local quantity computable by the dynamic
program, and on the path one trains it vanishes identically.

\begin{theorem}[Tropical performance difference lemma]
\label{thm:pdl}
For actors $\pi,\pi'$ and a state $s$ with $V^{\pi'}(s)>-\infty$, let
$\tau^\star_{\pi'}(s)$ be a path attaining $V^{\pi'}(s)$. Then
\begin{equation}
V^{\pi'}(s)-V^{\pi}(s)=\sum_{(s_t,a_t)\in\tau^\star_{\pi'}(s)}\Big[A^{\pi}(s_t,a_t)+\lambda\log\frac{\pi'(a_t\mid s_t)}{\pi(a_t\mid s_t)}\Big].
\label{eq:pdl}
\end{equation}
\end{theorem}

\begin{proof}
Rewards and terminal scores do not depend on the actor, so for any path
$W_{\pi'}(\tau')=W_\pi(\tau')+\lambda\sum_t\log\frac{\pi'(a_t\mid s_t)}{\pi(a_t\mid s_t)}$.
Apply this to $\tau'=\tau^\star_{\pi'}(s)$, which by
Lemma~\ref{lem:pathid}(a) satisfies $W_{\pi'}(\tau')=V^{\pi'}(s)$ and has an
accepted terminal, and expand $W_\pi(\tau')$ by Lemma~\ref{lem:pathid}(b).
\end{proof}

Set the two lemmas side by side. Classically,
$V^{\pi'}(s)-V^\pi(s)=\mathbb E_{\tau\sim\pi'}[\sum_tA^\pi(s_t,a_t)]$: the old
advantage, averaged over the new policy's trajectories. Tropically the average
is over one path, and a second term appears, the log-ratio of the two actors
along it. It appears because in tropical RL the per-step reward is the actor's
own log-likelihood, so changing the actor changes the rewards; there is no
analogue in classical RL, where rewards are exogenous. The log-ratio that
trust-region methods introduce by hand as an importance weight is here part of
the identity, additively, along a single path.

\begin{corollary}[Sandwich, and a witness for the slack]
\label{cor:sandwich}
Under the hypotheses of Theorem~\ref{thm:pdl}, with $\tau^\star_\pi(s)$ a
path attaining $V^\pi(s)$,
\begin{equation}
\lambda\sum_{(s_t,a_t)\in\tau^\star_\pi(s)}\log\frac{\pi'(a_t\mid s_t)}{\pi(a_t\mid s_t)}
\ \le\ V^{\pi'}(s)-V^\pi(s)\ \le\
\lambda\sum_{(s_t,a_t)\in\tau^\star_{\pi'}(s)}\log\frac{\pi'(a_t\mid s_t)}{\pi(a_t\mid s_t)}.
\label{eq:sandwich}
\end{equation}
The slack of the left inequality is
$W_{\pi'}(\tau^\star_{\pi'}(s))-W_{\pi'}(\tau^\star_\pi(s))\ge0$, the amount by
which the new witness beats the old witness under the new actor; the slack of
the right inequality is $-\sum_{\tau^\star_{\pi'}(s)}A^\pi(s_t,a_t)\ge0$, the
old-actor deficit of the new witness.
\end{corollary}

\begin{proof}
Lower bound: by Lemma~\ref{lem:pathid}(a),
$V^{\pi'}(s)\ge W_{\pi'}(\tau^\star_\pi(s))=W_\pi(\tau^\star_\pi(s))+\lambda\sum_{\tau^\star_\pi(s)}\log\frac{\pi'}{\pi}
=V^\pi(s)+\lambda\sum_{\tau^\star_\pi(s)}\log\frac{\pi'}{\pi}$, and the slack
is $V^{\pi'}(s)-W_{\pi'}(\tau^\star_\pi(s))=W_{\pi'}(\tau^\star_{\pi'}(s))-W_{\pi'}(\tau^\star_\pi(s))$,
nonnegative since $\tau^\star_{\pi'}(s)$ attains the maximum under $\pi'$.
Upper bound: drop the non-positive advantages in~\eqref{eq:pdl}; the slack is
their negative sum.
\end{proof}

The lower bound is what TROPIC's improvement step maximizes with
$\pi=\bar\pi_t$: the log-likelihood gain on the current witness. It is a valid
lower bound on the true improvement for \emph{any} $\pi'$, with no penalty
term. A penalty-free bound of this kind also exists for the hull distance,
where it is the evidence lower bound; what is specific to the vertex distance
is that both ends of the sandwich are single-path log-likelihood ratios and
that the slack is not a divergence between posteriors but a concrete path, the
new witness. When discovery or composition has made a new path the best one
under the new actor, that path is the slack, and the slack is the contribution
of composition to improvement. Nothing in the classical lemma isolates this
quantity.

\begin{corollary}[Tropical policy gradient, and the classical weights]
\label{cor:pg}
Let the actor be tabular softmax and let the maximizer $\tau^\star$ of
$V^{\pi_z}(s_0)$ be unique. Then
\begin{equation}
\frac{\partial V^{\pi_z}(s_0)}{\partial z_{s,a}}=
\begin{cases}
\lambda\big(\mathbf 1\{a=a^\star(s)\}-\pi_z(a\mid s)\big), & s\in S_{\tau^\star},\\[2pt]
0, & \text{otherwise},
\end{cases}
\label{eq:troppg}
\end{equation}
where $a^\star(s)$ is the action of $\tau^\star$ at $s$. For comparison, with
zero internal reward and terminal reward $R\in\{0,1\}$, the expected-return
value $\VER(s_0)=\sum_\tau P_z(\tau)R(\tau)$ satisfies
\begin{equation}
\frac{\partial\VER(s_0)}{\partial z_{s,a}}=\dvis(s)\,\pi_z(a\mid s)\,A_{\mathrm{ER}}(s,a),
\qquad
A_{\mathrm{ER}}(s,a)=\VER(T(s,a))-\VER(s),
\label{eq:classicalpg}
\end{equation}
where $\dvis(s)$ is the probability that a rollout from $s_0$ visits $s$.
\end{corollary}

\begin{proof}
For~\eqref{eq:troppg}, Corollary~\ref{cor:maxconseq}(b) applies to
$W_z(\tau)=\sum_t[\lambda\log\pi_z(a_t\mid s_t)+r(s_t,a_t)]+g(y)$, which is
continuously differentiable in $z$, and gives
$\nabla_zV^{\pi_z}(s_0)=\lambda\sum_{(s_t,a_t)\in\tau^\star}\nabla_z\log\pi_z(a_t\mid s_t)$.
Since $\partial\log\pi_z(a_t\mid s_t)/\partial z_{s,a}=\mathbf 1\{s=s_t\}\big(\mathbf 1\{a=a_t\}-\pi_z(a\mid s)\big)$
and a path visits each state at most once, the stated formula follows.

For~\eqref{eq:classicalpg}, $\partial\log P_z(\tau)/\partial z_{s,a}=
\mathbf 1\{s\in S_\tau\}\big(\mathbf 1\{a_\tau(s)=a\}-\pi_z(a\mid s)\big)$, so
\[
\frac{\partial\VER(s_0)}{\partial z_{s,a}}=\sum_{\tau\ni s}P_z(\tau)R(\tau)\big(\mathbf 1\{a_\tau(s)=a\}-\pi_z(a\mid s)\big).
\]
Every path through $s$ factors as a prefix from $s_0$ to $s$, an action at
$s$, and a suffix from $T(s,a_\tau(s))$, with
$P_z(\tau)=P_z(\text{prefix})\pi_z(a_\tau(s)\mid s)P_z(\text{suffix})$ and
$R(\tau)$ determined by the suffix. Summing over prefixes gives $\dvis(s)$ and
summing over suffixes gives $\VER$ of the successor, so
$\sum_{\tau\ni s,\,a_\tau(s)=a}P_z(\tau)R(\tau)=\dvis(s)\pi_z(a\mid s)\VER(T(s,a))$
and $\sum_{\tau\ni s}P_z(\tau)R(\tau)=\dvis(s)\VER(s)$. Substituting,
\[
\frac{\partial\VER(s_0)}{\partial z_{s,a}}=\dvis(s)\pi_z(a\mid s)\VER(T(s,a))-\pi_z(a\mid s)\dvis(s)\VER(s)
=\dvis(s)\pi_z(a\mid s)A_{\mathrm{ER}}(s,a).
\]
\end{proof}

The two formulas differ in exactly the way the two performance difference
lemmas differ. The classical gradient at $(s,a)$ carries the probability of
reaching $s$, the probability of taking $a$, and, inside $A_{\mathrm{ER}}$, the
probability of completing the task afterwards. On a path whose steps each have
probability $\rho$, every one of these factors is a power of $\rho$ and the
gradient at every logit on the path is of order $\rho^d$. The tropical gradient
on the witness carries none of them: it is $\lambda(\mathbf 1-\pi)$, of order
one however improbable the path. Natural policy gradient removes the upstream
factor by preconditioning and keeps the completion factor; the tropical
gradient removes both, because the path it differentiates is certified and its
advantage is identically zero. The next theorem turns this into a rate
separation on the simplest instance, with exact gradients, so that neither
sampling noise nor value estimation is involved.

\begin{theorem}[Exponential versus logarithmic escape on a chain]
\label{thm:plateau}
Let the graph be a chain $s_0\to s_1\to\dots\to s_d$, where at each $s_k$
($k<d$) the action $a^\star$ leads to $s_{k+1}$ and the action $a'$ leads to a
rejected terminal, and $s_d$ is an accepted terminal. Let the actor be tabular
softmax with one free logit per state, $\pi_k=\pi(a^\star\mid s_k)=\sigma(z_k)$,
initialized at $\pi_k^{(0)}=\rho\le\tfrac12$ for all $k$, and take
$\lambda=1$, $r\equiv0$. Consider gradient ascent with step $\eta>0$ on the
logits.
\begin{enumerate}
\item[(a)] \textnormal{(Expected return.)} For $J_{\mathrm{ER}}(z)=\VER(s_0)=\prod_{k<d}\pi_k$,
the first iteration $N$ at which $J_{\mathrm{ER}}\ge2\rho^d$ satisfies
\[
N>\frac{\log2}{2\,d\,\eta\,\rho^{d}}.
\]
\item[(b)] \textnormal{(Tropical.)} For $J_{\mathrm{trop}}(z)=\Vtrop(s_0)=\sum_{k<d}\log\pi_k$,
after $N_0=\lceil\tfrac{2}{\eta}\log\tfrac1\rho\rceil$ iterations every
$\pi_k\ge\tfrac12$, so $J_{\mathrm{trop}}\ge-d\log2$, and for every $n>N_0$,
\[
-J_{\mathrm{trop}}^{(n)}<\frac{2d}{\eta\,(n-N_0)}.
\]
\end{enumerate}
The bound in (a) is exponential in $d$; the bound in (b) is logarithmic in
$1/\rho$ and independent of $d$, followed by a $1/n$ rate.
\end{theorem}

\begin{proof}
(a) Since $\log\pi_k=\log\sigma(z_k)$ and $\frac{d}{dz}\log\sigma(z)=1-\sigma(z)$,
\[
\frac{\partial J_{\mathrm{ER}}}{\partial z_k}=J_{\mathrm{ER}}\,\frac{\partial\log J_{\mathrm{ER}}}{\partial z_k}=J_{\mathrm{ER}}(1-\pi_k)\in(0,J_{\mathrm{ER}}).
\]
Hence each $z_k$, each $\pi_k$ and $J_{\mathrm{ER}}$ are non-decreasing along
the iteration, and at any iteration $n$ with $J^{(n)}_{\mathrm{ER}}<2\rho^d$ every
logit moves by less than $2\eta\rho^d$. The function $\log\sigma$ is concave,
so for $\Delta>0$, $\log\sigma(z+\Delta)-\log\sigma(z)\le\Delta\,(1-\sigma(z))\le\Delta$.
Therefore, at any such iteration,
\[
\log J^{(n+1)}_{\mathrm{ER}}-\log J^{(n)}_{\mathrm{ER}}=\sum_{k<d}\big[\log\pi^{(n+1)}_k-\log\pi^{(n)}_k\big]<2\,d\,\eta\,\rho^d.
\]
All iterations before $N$ have $J_{\mathrm{ER}}<2\rho^d$, so summing,
$\log J^{(N)}_{\mathrm{ER}}-\log J^{(0)}_{\mathrm{ER}}<2d\eta\rho^dN$. Since
$J^{(0)}_{\mathrm{ER}}=\rho^d$ and $J^{(N)}_{\mathrm{ER}}\ge2\rho^d$, the left
side is at least $\log2$, which gives the claim.

(b) Here $\partial J_{\mathrm{trop}}/\partial z_k=1-\pi_k=\sigma(-z_k)$, so
$z^{(n+1)}_k=z^{(n)}_k+\eta\,\sigma(-z^{(n)}_k)$ and each $z_k$ is increasing.
The coordinates evolve independently and identically, so fix $k$.

Phase one. While $z^{(n)}_k<0$ we have $\sigma(-z^{(n)}_k)>\tfrac12$, so
$z^{(n+1)}_k-z^{(n)}_k>\eta/2$. The initial logit is
$z^{(0)}_k=\log\frac{\rho}{1-\rho}\ge\log\rho=-\log\frac1\rho$, since
$\frac{\rho}{1-\rho}\ge\rho$. Hence $z_k$ reaches $0$ within at most
$\lceil\frac{\log(1/\rho)}{\eta/2}\rceil=N_0$ iterations, and, being
increasing, stays nonnegative thereafter. Thus $\pi^{(n)}_k\ge\tfrac12$ for
all $n\ge N_0$ and $J^{(N_0)}_{\mathrm{trop}}\ge-d\log2$.

Phase two. For $z\ge0$, $1+e^{z}\le2e^{z}$, so
$\sigma(-z)=\frac{1}{1+e^{z}}\ge\tfrac12e^{-z}$. Let $y_n=e^{z^{(n)}_k}$ for
$n\ge N_0$, so $y_{N_0}\ge1$. Using $e^{x}\ge1+x$,
\[
y_{n+1}=y_ne^{z^{(n+1)}_k-z^{(n)}_k}\ge y_n\big(1+\eta\,\sigma(-z^{(n)}_k)\big)\ge y_n\Big(1+\frac{\eta}{2y_n}\Big)=y_n+\frac{\eta}{2}.
\]
Hence $y_n\ge1+\frac{\eta}{2}(n-N_0)$ for $n\ge N_0$. Finally
$-\log\pi^{(n)}_k=\log(1+e^{-z^{(n)}_k})\le e^{-z^{(n)}_k}=1/y_n\le\big(1+\tfrac{\eta}{2}(n-N_0)\big)^{-1}<\frac{2}{\eta(n-N_0)}$
for $n>N_0$. Summing over the $d$ coordinates gives the bound.
\end{proof}

\begin{remark}[Where the separation lies]
On this chain there is a single certified path, so
$J_{\mathrm{trop}}=\log J_{\mathrm{ER}}$ exactly, and part (b) also describes
gradient ascent on $\log J_{\mathrm{ER}}$, which is the maximum-likelihood
objective of MaxRL. The separation in Theorem~\ref{thm:plateau} is therefore
between expected return and its logarithm, and the obstruction in (a) is the
factor $J_{\mathrm{ER}}$ in the gradient, which is the weight
$\omega_{\mathrm{ER}}=P_\theta(\tau)$ of Proposition~\ref{prop:weights}. What
distinguishes the tropical objective from maximum likelihood is not this rate
but the object it moves toward, as established in
Section~\ref{app:projections} and in Proposition~\ref{prop:incompat}: on a
problem with several certified paths the maximum-likelihood gradient is spread
over them by posterior weight, while the tropical gradient is the score
function of one. The theorem is stated with exact gradients to make clear that
the exponential cost in (a) is a property of the objective's landscape, not of
sampling; no critic, baseline or additional data removes it. With sampled
gradients every expected-return method also pays $\rho^{-d}$ to observe the
path at all, and that cost is the one the archive removes.
\end{remark}

\subsection{How TROPIC realizes these results}
\label{app:tropic}

The algorithm of Section~\ref{sec:tropic} is the loop these results describe.
The archive is the set of certified vertices of Section~\ref{app:projections};
discovery adds vertices and composition adds vertices assembled from existing
ones, and by Theorem~\ref{thm:mono} the set never shrinks. Rescoring under the
frozen actor recomputes the distances from the current actor to each vertex.
Corollary~\ref{cor:maxconseq}(a) is why
the prefix and suffix passes may be run separately and joined. The
improvement step is maximum likelihood on that nearest vertex, which
Corollary~\ref{cor:maxconseq}(b) identifies as the exact gradient of the
tropical value, Corollary~\ref{cor:sandwich} identifies as maximizing a
penalty-free lower bound on the true improvement, and Theorem~\ref{thm:cpi}
shows to be a pointwise improvement step on the whole retained graph. By
Theorem~\ref{thm:freqfree} and Proposition~\ref{prop:weights} the step depends
on the archive only through its support, which is why acquisition may be as
off-policy as the frontier rule makes it, with no correction. By
Proposition~\ref{prop:incompat} it is not a step that expected-return
information could have selected. What the actor is pulled toward, at every
iteration and at every interface, is the nearest computation it has been shown
to be capable of; what it is never pulled toward is a distribution over
computations, and what it is never allowed to lose is a computation it has
already been shown.

%% file: appendix/appendix_environments.tex
\section{RL Environments}
\label{app:RL_environments}

\begin{figure}[htbp]
  \centering
  \includegraphics[width=\linewidth]{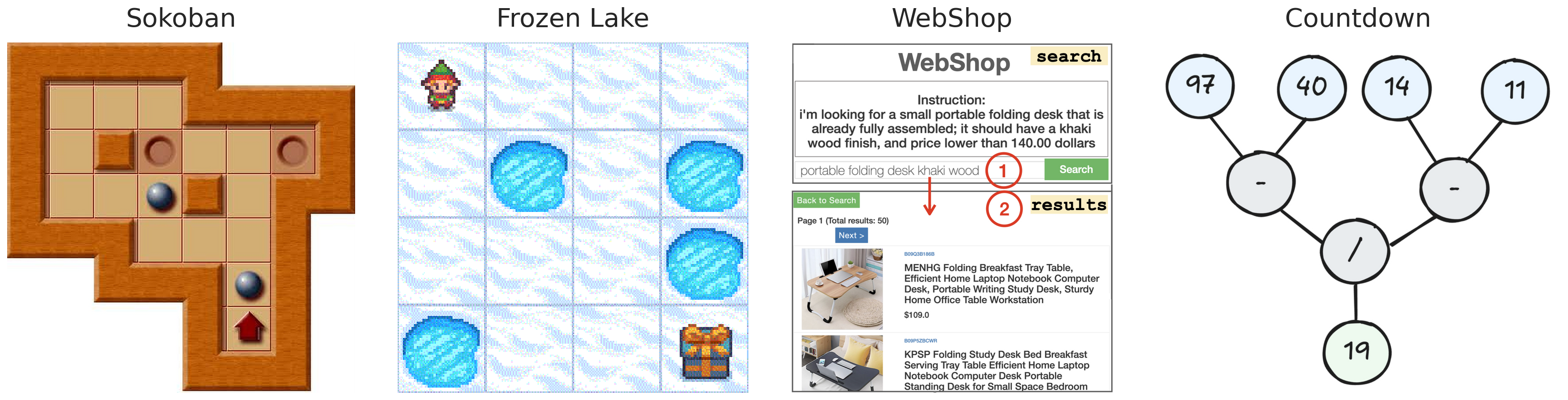}
  \vspace{-1em}
  \caption{\textbf{Environments.} From left to right: \textbf{Sokoban}, push the box onto the target without creating a dead end; \textbf{FrozenLake}, reach the goal without stepping into a hole; \textbf{WebShop}, search for and purchase a product that matches a natural-language instruction; and \textbf{Countdown}, combine the given numbers with arithmetic operations to reach the target, here $(97-40)/(14-11)=19$.}
  \vspace{-1em}
  \label{fig:Environments}
\end{figure}

All four environments in ~\refig{fig:Environments} are multi-turn agentic tasks from RAGEN-2~\citep{wang2026ragen2reasoningcollapseagentic}. 
In each, the agent observes a text rendering of the current environment state, emits an action, and receives a binary verifier signal only at the end of the episode. 
As described in the main body, \textsc{Tropic} prompts the actor with the problem and the current environment state only, and includes the step count in the state so that the graph of states is acyclic. 
Two attempts therefore share a state whenever they reach the same environment configuration at the same step, even if they got there by different actions. 
These shared states are the interfaces at which \textsc{Tropic} joins a prefix from one attempt with a suffix from another. 
For every task, we describe below what a state is, what the verifier checks, and why such interfaces arise.

\paragraph{Sokoban.}
The agent moves on a $6\times6$ grid and must push a single box onto a target cell. 
A state is the board configuration, i.e., the positions of the agent and the box. 
An episode succeeds when the box rests on the target. 
Pushing the box against a wall or into a corner can make the puzzle unsolvable, so many attempts fail late after a sequence of correct early pushes. 
For expected-return methods, such attempts yield zero reward and their correct prefixes receive no credit. 
\textsc{Tropic} keeps these prefixes in its graph: if another attempt, started from the same board configuration, reaches the target, the two are joined into a verified solution that neither attempt produced alone. 
Because different move orders often lead to the same board, such shared configurations are frequent.

\paragraph{Countdown.}
The agent is given three to six input numbers and a target, and must combine the numbers with $+$, $-$, $\times$, and $\div$, using each number exactly once, to reach the target. 
A state is the multiset of numbers still available, and each step replaces two numbers by the result of an operation on them. 
The verifier checks that the final number equals the target. 
Because arithmetic steps on disjoint numbers commute, different attempts frequently reach the same intermediate multiset, e.g., $\{57, 3\}$ from $97-40$ and $14-11$ in either order. 
An attempt that computes a useful intermediate result but then fails, and another that finishes correctly from the same intermediate multiset, can therefore be joined into a verified derivation. 
Training and evaluation use disjoint problem splits.

\paragraph{FrozenLake.}
The agent navigates a $4\times4$ grid from a start cell to a goal cell while avoiding holes, which end the episode. 
Each cell is safe with probability $0.9$ when the map is generated, and transitions are deterministic in our experiments. 
A state is the agent's position. Success requires an entire sequence of safe moves, and a single misstep anywhere ends the episode with zero reward, so failed attempts often contain long safe prefixes. 
\textsc{Tropic} retains these prefixes and restarts new attempts from the frontier of safe cells, so a safe route to the middle of the grid found in one attempt can be completed by a suffix found in another.

\paragraph{WebShop.}
The agent receives a natural-language shopping instruction, e.g., a portable folding desk with a given finish below a price limit, and must search the catalog, navigate result and product pages, select the required attributes, and purchase a matching item. 
We use the RAGEN-2 \texttt{webshop-minimal} adaptation. 
A state is the current page together with the selections made so far. 
An episode succeeds only if the purchased item satisfies all requested attributes. 
Episodes are long, and a mistake late in the episode, such as a wrong attribute, discards an otherwise correct search and navigation. 
\textsc{Tropic} keeps the correct search and navigation steps and can join them with a correct selection-and-purchase suffix found in another attempt from the same product page.

\paragraph{Why these Tasks Favor Composition.}
The four environments share the structure that makes expected return inefficient and tropical composition effective: the reward is sparse and terminal, failures typically occur after many correct steps, and different attempts pass through the same intermediate states. 
On-policy baselines must sample a complete successful episode before any of its steps is reinforced. 
\textsc{Tropic} instead retains every valid step, including steps from failed attempts, and assembles verified solutions at shared states. 
This is the mechanism measured directly by \textsc{First-Solve} in~\refsec{sec:ablation_results}.

%% file: appendix/appendix_training_setup.tex
\section{Training and Evaluation Setup}
\label{app:training_and_evaluation_setup}

Our implementation builds on RAGEN-2~\citep{wang2026ragen2reasoningcollapseagentic} and its VERL training backend~\citep{verl}. 
We follow the default RAGEN-2 configuration where applicable and describe all task- and method-specific settings below.

\paragraph{Models.}
We initialize all policies from Qwen2.5-3B-Instruct~\citep{qwen2025qwen25technicalreport}, with $3.1$B parameters, and Gemma-4-E4B-it~\citep{gemmateam2026gemma4}, with $4$B effective and $8$B total parameters including embeddings, i.e., more than twice as many parameters in total.
The two models also differ in family, architecture, and pretraining data, which lets us test whether \textsc{Tropic}'s gains depend on a particular model.
All models are fully fine-tuned without parameter-efficient adapters.

\paragraph{Baselines.}
We compare against PPO~\citep{schulman2017proximalpolicyoptimizationalgorithms}, GRPO~\citep{shao2024deepseekmath}, DAPO~\citep{yu2025dapoopensourcellmreinforcement}, and RAGEN-2's SNR-Aware Filtering~\citep{wang2026ragen2reasoningcollapseagentic}. 
All use AdamW with an actor learning rate of $10^{-6}$, $(\beta_1,\beta_2)=(0.9,0.999)$, and weight decay $0.01$. 
PPO and SNR-Aware use a learned critic with learning rate $10^{-5}$ and Generalized Advantage Estimation~\citep{schulman2018highdimensionalcontinuouscontrolusing} with $\gamma=1$ and $\lambda_{\mathrm{GAE}}=1$. 
GRPO and DAPO use group-relative advantages without a critic. 
For DAPO we use the standard implementation provided by RAGEN-2. 
The policy-ratio clipping interval is $[0.8, 1.28]$, entropy regularization uses coefficient $0.001$, and no reference-policy KL penalty is applied. 
SNR-Aware Filtering applies the RAGEN-2 rollout-group filter with cumulative probability threshold $p=0.9$, using the within-group reward standard deviation as its selection score. 
The other baselines apply no rollout filtering. 
The additional methods in Table~\ref{tab:sokoban_ablation} (MaxRL, GiGPO, and GiGPO\,+\,T-STAR) use the same backbone, optimizer, and rollout budget, with the remaining hyperparameters set as recommended by their authors.

\paragraph{\textsc{Tropic}.}
\textsc{Tropic} uses the same AdamW optimizer and actor learning rate ($10^{-6}$) as the baselines, and performs maximum-likelihood updates on verified paths.
It uses no critic, policy-ratio clipping, entropy regularization, or reference-policy KL penalty.
Each active training problem receives $G=2$ sequential groups of $n=8$ attempts per iteration.
The first group starts from the initial state; subsequent groups start from a selected frontier state when available, otherwise falling back to the initial state.
The resulting $G n=16$ rollout starts per problem match the baselines' nominal attempt budget; this does not imply equal generated-token counts or verification costs.
Valid transitions are retained up to the graph capacity even when the originating attempt fails, but only complete paths replayed from the initial state and accepted by the verifier are eligible as training targets.
Graph memory persists across iterations over a recurring pool of 128 training problems, with eight problems visited per iteration.
Table~\ref{tab:tropic_hparams} lists the \textsc{Tropic}-specific settings for the main experiments.

\begin{table}[h]
    \centering
    \caption{\textbf{\textsc{Tropic} hyperparameters.}
    Shared across tasks and models unless stated otherwise.}
    \label{tab:tropic_hparams}
    \vspace{0.06in}
    \small
    \ccfont
    \setlength{\tabcolsep}{3.8pt}
    \renewcommand{\arraystretch}{1.1}
    \resizebox{0.7\textwidth}{!}{%
    \begin{tabular}{l l c}
    \toprule
    \textbf{Hyperparameter} & \textbf{Description} & \textbf{Value} \\
    \midrule
    $G$ & Attempt groups per problem/iteration & $2$ \\
    $n$ & Attempts per group & $8$ \\
    $\kappa$ & Frontier exploration weight & $1.0$ \\
    --  & Training-rollout temperature & $1.0$ \\
    \midrule
    $L$ & Prefixes and suffixes per shared state & $4$ each \\
    --  & Composition candidates per problem/refresh & $64$ \\
    \midrule
    --  & Maximum retained edges per problem & $1{,}024$ \\
    --  & Maximum archived solutions per problem & $64$ \\
    --  & Edge-rescoring batch size & $64$ \\
    \midrule
    $K$ & Maximum training paths per problem & $2$ \\
    --  & Coverage-based path selection & Enabled \\
    --  & Update epochs per iteration & $2$ \\
    --  & Recurring training-problem pool size & $128$ \\
    \bottomrule
    \end{tabular}
    }
\end{table}

Composition is performed after each attempt group, giving two refreshes per iteration and at most $2 \times 64 = 128$ composition replays per active problem per iteration.
This limit does not include verification of sampled successful paths or any additional restart-prefix checks.

\paragraph{Training and Infrastructure.}
All experiments run on 8 NVIDIA A100 GPUs using Fully Sharded Data Parallel (FSDP) training with BF16 mixed precision. 
Rollouts are generated with vLLM, and distributed execution is coordinated through Ray. 
The optimization mini-batch size is 32 sequences, i.e., four sequences per GPU, with gradient accumulation where necessary. 
All methods train for at most 200 rollout--update iterations on Sokoban, Countdown, and FrozenLake, and at most 100 on WebShop.
Validation is performed before training and every 25 iterations.

\paragraph{Evaluation Protocol.}
Final evaluation uses a fixed set of 512 held-out instances per environment that are not seen during training. 
Decoding uses temperature $T=0.5$, top-$p=1$, and no top-$k$ truncation. 
At evaluation time, \textsc{Tropic} uses no archive, restarts, or composition: the trained policy must solve each instance in a single attempt, exactly like the baselines, so the score measures what training has transferred into the model. 
Each method is evaluated at its final checkpoint, and we report the mean and standard deviation across three independent training seeds.

\paragraph{Metrics.}
We report the following metrics.

\begin{enumerate}

    \item \textbf{Success rate.} The fraction of evaluation instances solved in a single attempt. Success means placing the box on the target in Sokoban, satisfying the arithmetic verifier in Countdown, reaching the goal in FrozenLake, and completing a fully correct purchase in WebShop.
    
    \item \textbf{pass@$k$.} The probability that at least one of $k$ attempts solves an evaluation instance (see \refig{fig:teaser}).
    
    \item \textbf{Generated tokens per episode.} The number of tokens generated across all turns of an evaluation episode, averaged over successful and unsuccessful episodes. We use each model's tokenizer and exclude prompts, environment observations, and padding. This metric measures inference cost only.
    
    \item \textbf{Training wall-clock time.} The elapsed time from the start to the end of training, including rollout generation, optimization, scheduled validation, and checkpointing, and excluding external environment setup and checkpoint preparation.
    
    \item \textbf{\textsc{First-Solve}} (ablations only). The fraction of solved training instances whose first verified solution was assembled by composition before any complete successful trajectory had been sampled. By definition, it is zero for \emph{No Composition}, which does not compose, and for \emph{Successful-Only}, which composes only from already successful trajectories.

\end{enumerate}

\paragraph{What the Evaluation Measures.}
We use \emph{compositional} as described in the main body: a solution is compositional when its fragments have been produced in separate attempts but never together in a single one. 
\textsc{Tropic} targets exactly this situation during training, and \textsc{First-Solve} measures it directly. 
On Sokoban, $31.6\%$ of solved training problems are first solved by a composition never sampled end to end. 
The held-out evaluation asks a separate question: whether the policy trained on these assembled solutions solves new instances on its own, in one attempt, without any of \textsc{Tropic}'s search machinery. 
Since every held-out instance is unseen during training, success requires producing a new solution rather than recalling a stored one. 
We do not claim systematic generalization in the sense of \citet{lake2018generalization}, i.e., to problems drawn from a different distribution of compositions, and leave such splits to future work.

\paragraph{On the Fairness of the Comparison.}
\textsc{Tropic} differs from the baselines in three respects that could affect the comparison: it restarts attempts from frontier states, it replays and verifies composed paths, and it uses its own stopping rule. We address each in turn.

\emph{Frontier restarts.} 
Restarting from a frontier state requires a deterministic, resettable environment. 
We regard this as the setting \textsc{Tropic} is designed for, not an advantage granted to it. To check that restarts alone do not explain the gains, the ablations in Table~\ref{tab:sokoban_ablation} separate them from composition. 
\textsc{Tropic} (No Composition) keeps the archive and frontier restarts and trains by maximum likelihood on sampled verified paths, which makes it closely related to archive-based exploration combined with rejection-sampling fine-tuning. 
It reaches $40.8\%$ on Sokoban, on par with the strongest baseline (SNR-Aware, $42.2\%$). 
Conversely, \textsc{Tropic} (Root-Only) removes frontier restarts but keeps composition, and reaches $51.8\%$, above every baseline. 
The improvement therefore comes from composing fragments across attempts, not from restarts.

\emph{Replay and verification.} 
Replaying a composed path executes its stored actions in the environment and queries the verifier. 
Accordingly, it generates no tokens. 
Since all methods generate the same number of rollouts per iteration, replay adds only environment calls, which are inexpensive relative to generation in all four tasks.

\emph{Stopping and Wall-Clock Time.}
All methods share the same iteration cap.
The baselines stop according to RAGEN-2's criteria, while \textsc{Tropic} stops early when no training problem yields a new update.
Reported wall-clock times therefore reflect the end-to-end cost of each method under its standard protocol, not a comparison at matched compute.
This does not favor \textsc{Tropic}. 
Since no method can exceed the shared cap, early stopping can only shorten \textsc{Tropic}'s training, never extend it, and each baseline trains for as long as its standard protocol allows. 
Running every method under its standard protocol is also the setting in which a practitioner would compare them.
Success rates are unaffected by this difference: all methods are evaluated at their final checkpoint on the same held-out instances.

%% file: appendix/appendix_limitations.tex
\section{Limitations}
Tropical RL changes a single operation of reinforcement learning: alternative trajectories are combined by a maximum instead of a sum. This one change gives the value three properties that expected return lacks. It names a concrete path, it never decreases as new verified paths are found, and it splits into a best prefix and a best suffix at any shared state. \textsc{Tropic} uses these properties to assemble solutions from fragments the actor never produced together. On four agentic tasks, this translates into higher single-attempt success than on-policy baselines.

It is worth being precise about the guarantee. The archive, and the capability it certifies, never shrink. For an actor with separate parameters per state, one step on the best path raises the tropical value of every state in the graph (Theorem~\ref{thm:cpi-main}). A language model, however, shares parameters across states and problems, so a step on one problem can still lower the likelihood of a verified path on another. The coverage paths of Section~\ref{sec:tropic} mitigate this, but they are a heuristic, not a proof. \textsc{Tropic} therefore guarantees that no verified computation is lost from \emph{memory}, not that the \emph{actor} never forgets.

This distinction points to the direction we find most promising: continual self-improvement, in which a capability, once verified, is kept and becomes reusable by later ones. Expected return has no mechanism for this, since it stores discoveries only in current probabilities and normalization makes them compete (Proposition~\ref{prop:forget}). The tropical value provides the missing bookkeeping: a set of verified computations that only grows and whose best elements compose. Whether this record can be transferred into shared weights without interference, at scale, remains open.